\documentclass[sn-mathphys,iicol]{sn-jnl}

\usepackage{amsmath}
\usepackage{amssymb}
\usepackage{graphicx}
\usepackage[Symbol]{upgreek}
\usepackage{epstopdf}
\usepackage{subcaption}
\usepackage{enumerate}
\usepackage{paralist}

\usepackage[normalem]{ulem}

\usepackage{mathtools}

\usepackage{float}

\usepackage{latexsym}
\usepackage{multicol}
\usepackage{multirow}
\usepackage{lipsum}

\usepackage{makecell}
\usepackage{breqn}

\usepackage{verbatim}

\usepackage{hyperref}

\usepackage[normalem]{ulem}

\usepackage{color}
\usepackage{xcolor}
\usepackage{bm}
\usepackage[font=small,labelfont=bf]{caption}

\usepackage{booktabs}

\newcommand{\vect}[1]{\bm{#1}}

\newcommand{\phVar}[0]{s}
\newcommand{\fterm}[0]{\vect{f}_{\phVar}}

\newcommand{\bvect}[1]{\bar{\vect{#1}}}
\newcommand{\hvect}[1]{\hat{\vect{#1}}}
\newcommand{\dvect}[1]{\dot{\vect{#1}}}
\newcommand{\ddvect}[1]{\ddot{\vect{#1}}}

\newcommand{\vphi}[0]{\vect{\phi}}
\newcommand{\vPhi}[0]{\vect{\Phi}}

\newcommand{\ddhvect}[1]{\ddot{\hvect{#1}}}

\newcommand{\diag}[1]{\text{diag}\left(#1\right)}
\newcommand{\norm}[1]{\left|\left|{#1}\right|\right|}
\newcommand{\trace}[1]{\text{tr}\{{#1}\}}

\newcommand{\cth}[0]{\cos(\theta_2)}
\newcommand{\sth}[0]{\sin(\theta_2)}

\newcommand{\partDer}[2]{\frac{\partial #1}{\partial #2}}
\newcommand{\partDerSq}[2]{\frac{\partial^2 #1}{\partial^2 #2}}

\definecolor{lightbrown}{rgb}{0.85, 0.33, 0.1}

\newcommand{\update}[1]{#1}

\newcommand{\theoref}[1]{\textit{Theorem} \ref{#1}}
\newcommand{\remarkref}[1]{\textit{Remark} \ref{#1}}

\newcommand{\stkout}[1]{\ifmmode\text{\sout{\ensuremath{#1}}}\else\sout{#1}\fi}

\newcommand\scalemath[2]{\scalebox{#1}{\mbox{\ensuremath{\displaystyle #2}}}}

\usepackage[normalem]{ulem}
\usepackage{url}
\DeclareUrlCommand\ULurl{%
  \renewcommand\UrlLeft{\uline\bgroup}%
  \renewcommand\UrlRight{\egroup}}

\jyear{2023}%

\theoremstyle{thmstyleone}%
\newtheorem{theorem}{Theorem}
\theoremstyle{thmstyletwo}%
\newtheorem{remark}{Remark}%

\theoremstyle{thmstylethree}%

\begin{document}

\title[Article Title]{
Dynamic via-points and improved spatial generalization for online trajectory planning with Dynamic Movement Primitives
}


\author*[1]{\fnm{Antonis} \sur{Sidiropoulos}}\email{antosidi@ece.auth.gr}


\author[1]{\fnm{Zoe} \sur{Doulgeri}}\email{doulgeri@ece.auth.gr}

\affil*[1]{\orgdiv{Department of
Electrical and Computer Engineering}, \orgname{Aristotle University of Thessaloniki}, \orgaddress{\street{Panepistimioupoli}, \city{Thessaloniki}, \postcode{54124}, \country{Greece}}}




\abstract{
Dynamic Movement Primitives (DMP) have found remarkable applicability and success in various robotic tasks, which can be mainly attributed to their generalization, modulation and robustness properties. Nevertheless, the spatial generalization of DMP can be problematic in some cases, leading to excessive or unnatural spatial scaling. Moreover, incorporating intermediate points (via-points) to adjust the DMP trajectory, is not adequately addressed. In this work we propose an improved online spatial generalization, that remedies the shortcomings of the classical DMP generalization, and moreover allows the incorporation of dynamic via-points. This is achieved by designing an online adaptation scheme for the DMP weights which is proved to minimize the distance from the demonstrated acceleration profile to retain the shape of the demonstration, subject to dynamic via-point and initial/final state constraints. Extensive comparative simulations with the classical and other  DMP variants are conducted, while experimental results validate the applicability and efficacy of the proposed method.
}




\maketitle


\section{Introduction} \label{sec:Introduction}

\noindent Dynamic Movement Primitives (DMP) \cite{Ijspeert2013, PASTOR2013351} have emerged as a promising method for encoding a desired trajectory and generalizing it to new situations. It is particularly favoured thanks to its ease of training, generalization, robustness and on-line modulation properties and has been applied in a plethora of practical robotic tasks and applications \cite{Pastor_2011_DMP, Ude_Task_specific_DMP_2010, Umlauft_2014, Hitting_DMP_Peters_2013, Gams_DMP_force_2014, Luo2022}. The trajectory to be encoded is typically provided through PbD (Programming by Demonstration) \cite{PbD_Billard}, as an intuitive and efficient means of teaching human-skills to robots.

In DMP, the objective of generalization is to generate a trajectory from a new initial position to a new, possibly time-varying, target position, while also preserving the shape of the demonstrated trajectory. 
To this end, different DMP variations and/or spatial generalization mechanisms have been proposed with the most well-known being the classical DMP formulation from \cite{Ijspeert2013}. However, the classical DMP spatial scaling has some important shortcomings, i.e. over-scaling of the generalized trajectory when the demonstrated end-points (initial and final position) are close, failure to generate a motion when the new end-points during execution coincide and mirroring of the trajectory when the sign of the difference between the new end-points is different from the sign of the difference of demonstrated ones. To remedy these issues, the bio-inspired DMP formulation is proposed in \cite{Bio_DMP_2009}. Nevertheless, it exhibits poor generalization for targets that are not close to the demonstration, as highlighted in \cite{Mollifier_bio_DMP_2021}, where an improved version of the bio-inspired DMP is proposed to enhance its spatial generalization. A novel generalization to tackle the scaling issues of classical DMP is also proposed in \cite{Koutras_DMP_scaling}. While \cite{Mollifier_bio_DMP_2021}, \cite{Koutras_DMP_scaling} resolve the classical DMP scaling issues, they introduce over-scaling in trajectories with amplitudes bigger than the distance between the initial and final position, which is typical in many practical applications, like in packing applications where lifting an object and placing it down in a box is required. A different offline spatial generalization approach is proposed in \cite{DMP_via_opt_2015} based on minimizing a norm, that is learned from multiple demonstrations, solving an iterative optimization problem.

Few DMP works consider via points on-line in their spatial generalization. Via-points can be particularly useful to adjust a trajectory on-line, e.g. to avoid obstacles in a predictable way by specifying through the via-points how the robot should circumvent them \cite{VMP_2019} or to push aside other objects that obstruct reaching pathways to a desired target \cite{via_points_fruits}.
\update{
Moreover, via-points offer a more flexible way of online adjusting a DMP trajectory, in contrast to other approaches that propose to join separate DMPs \cite{nemec_ude_2012_DMP_join, Kulvicius_2012_DMP_join}. In the latter case, the trajectory segments need to be pre-specified, the complexity increases since multiple DMPs have to be defined to connect the via-points and also extra adjustments are required to transition smoothly from one DMP segment to the next without halting the motion \cite{Hitting_DMP_Peters_2013}.}
Via-points are considered off-line  
in a DMP-like variant in \cite{DMP_QP_2015} through offline optimization and in   \cite{DMP_SQP_2015}, which can include via-point constraints. 
So far, online incorporation of via-points  is performed with probabilistic DMP \cite{ProIP_2014}, by online adapting the DMP weights to the new forcing term values, based on the provided via-points. Calculating the new forcing term value at  via-points requires estimates of velocity and acceleration at the via-point, which can result in significant noise, that gets even more adverse, the sparser the via-points are arranged, as shown in \cite{ProMP_2017}.  
Another approach to include via-points is proposed in \cite{Weitschat_2018}, by introducing a time-varying attractor that shifts smoothly from one via-point to the next, until the final target \cite{Weitschat_2018}. Nevertheless, this does not ensure that the via-point will be reached and the distortions in the trajectory shape that this time-varying attractor incurs are not predictable.

Other movement primitives that can include via-points are the Probabilistic Movement Primitives (ProMP) \cite{ProMPs_Paraschos_2013, ProMPs_Paraschos_2018} and  the Kernelized Movement Primitives (KMP) \cite{KMP_2019}. However,  multiple demonstrations are required for training such primitives.
Moreover, they cannot generalize well away from the demonstrations and have higher computational complexity with multiple DoFs. Of course, these primitives have other advantages  compared to DMP and vice versa; choosing the most appropriate depends on the target application. In \cite{VMP_2019} the Via-point Movement Primitives (VMP) are introduced, that adopt ideas from DMP and ProMP, to achieve adaptation to via-points, and were shown to have better extrapolation capabilities compared to ProMP, being also able to be trained from a single demonstration.

All aforementioned works are either off-line, or those that are online, consider static via-points, and do not address the case of via-points that may change dynamically. Dynamic changes of via-points occur if via-points are defined with respect to the position of the target or an obstacle, which is perturbed during execution. 
For instance in a human-robot collaborative packing scenario a via-point may be defined on top of the target box which may be displaced by the human to a more ergonomic position for him.  The robot has to adjust on-line to this change.

In this work we propose an improved online spatial generalization for DMP, that 
does not exhibit
the shortcomings of the spatial scaling of the other DMP variants and can incorporate dynamic via-points.
We achieve this by exploiting the novel DMP formulation presented in our previous work \cite{Antosidi_Rev_DMP}
to derive an online adaptation scheme for the DMP weights, based on minimizing the distance from the learned acceleration profile under the equality constraints of via-points.
Moreover, we show how the proposed novel generalization can also be combined with our previous work \cite{Antosidi_DMP_constr_2022} to impose kinematic inequality constraints in order to generate feasible trajectories in the presence of kinematic bounds related to the robot and the task environment (obstacles and via points). 


The rest of this paper is organized as follows: In Section \ref{sec:DMP_prelim} we provide the preliminaries for the novel DMP formulation from \cite{Antosidi_Rev_DMP}. Section \ref{sec:Proposed_generalization} is devoted to the proposed novel online generalization. Simulations that demonstrate how the proposed method performs in comparison with the classical and other SoA DMP variants are provided in section \ref{sec:Simulations}.
The practical usefulness and efficacy of the proposed method is further showcased in the experimental Section \ref{sec:Experiments}. Finally, we draw the conclusions in Section \ref{sec:Conclusions}.
The source code for all conducted simulations and experiments can be found at
{\footnotesize \ULurl{github.com/Slifer64/novel_dmp_generalization.git}}.


\section{DMP preliminaries} \label{sec:DMP_prelim}

In this section we briefly present the DMP formulation introduced in our previous work \cite{Antosidi_Rev_DMP}, which we exploit in this work to derive an improved on-line spatial generalization that can also incorporate dynamic via-points. As shown in \cite{Antosidi_Rev_DMP}, this formulation is mathematically equivalent to the classical DMP formulation from \cite{Ijspeert2013}, retaining all properties of the latter. 

Consider the state $\vect{y} \in \mathbb{R}^n$, which can be joint positions, Cartesian position or Cartesian orientation expressed using the quaternion logarithm (see Appendix A for preliminaries on unit quaternions). 
The DMP consists of two sets of differential equations, the transformation and the canonical system.
A DMP can encode a desired trajectory $t_j, \ \vect{y}_{d,j}$ for $j=1...m$ where $m$ is the total number of points and $T_{f,d} = t_m - t_1$ the time duration of the demo, and through its transformation system it can generalize this trajectory spatially from an initial position $\vect{y}_0$ to a desired target position $\vect{g}$. 
Temporal scaling, i.e. execution of this trajectory with a different time duration or speed, is achieved through the canonical system, which provides the clock $s$ (time substitute) for the transformation system, to avoid direct time dependency.
The \textbf{transformation system} is given by:
\begin{equation} \label{eq:GMP_tf_sys}
    \ddvect{y} = \ddvect{y}_{s} - \vect{D}(\dvect{y}-\dvect{y}_{s}) - \vect{K}(\vect{y} - \vect{y}_{s} ) 
\end{equation} 
\begin{align}
    \vect{y}_{s}(s) &= \vect{K}_s (\vect{f}_p(s) - \hvect{y}_{d,0}) + \vect{y}_0 \label{eq:y_x_vec} \\
    \vect{f}_p(s) &= \vect{W}^T \vphi(s)\label{eq:fp_vec} \\
    \vect{K}_s &= \diag{(\vect{g} - \vect{y}_0) ./ (\hvect{g}_d - \hvect{y}_{d,0})} \label{eq:Ks_vec}
\end{align}
with $\vect{K}, \ \vect{D} \in \vect{\mathcal{S}}_{++}^n$, with $\vect{\mathcal{S}}_{++}^n$ denoting the set of symmetric $n \times n$ positive definite matrices,  
$\vect{y}_{s}$ is the scaled learned trajectory whose derivatives can be calculated analytically in closed form \cite{Antosidi_Rev_DMP}, the matrix $\vect{K}_s$ achieves the spatial scaling from the new initial position $\vect{y}_0$ to a new target $\vect{g}$, while $\hvect{g}_d \triangleq \vect{f}_p(s(T_{f,d}))$, $\hvect{y}_{d,0} \triangleq \vect{f}_p(s(0))$ are the learned initial and target demo positions. The demonstrated position trajectory is encoded through the weighted sum of Gaussians in $\vect{f}_p(s)$, where the DMP weights $\vect{W} \in \mathbb{R}^{K \times n}$ can be determined using Least Square, i.e. $\text{min}_{\vect{W}} \sum_{j=1}^{m} || \vect{y}_{d,j} - \vect{W}^T \vect{\phi}_j||_2^2$), or Locally Weighted Regression (LWR) \cite{Ijspeert2013}.
The Gaussian kernels are $\vphi(s)^T = [\psi_1(s) \ \cdots \ \psi_K(s)]/{\sum_{i=1}^{K} \psi_i(s)}$, with $\psi_i(s) = exp(-h_i(s-c_i)^2)$ with $c_i$ being the centers and $h_i$ the inverse widths. A reasonable heuristic is to set the centers $c_i$ equally spaced in $[0, \ 1]$ and set $h_i = 1 / (a_h(c_{i+1} - c_i))^2$ where $a_h$ controls the standard deviation of the Gaussian kernels. Choosing $a_h \in [1.2, \ 1.5]$ produces empirically good approximation results.

Given $\phVar(0) \triangleq \phVar_0$, $\phVar(\infty) \triangleq \phVar_f$, and considering without loss of generality  $\phVar_0=0$ and $\phVar_f=1$, the \textbf{canonical system} is given by:
\begin{equation} \label{eq:GMP_can_sys}
    \ddot{\phVar} = 
    \begin{cases}
        d_1(\dot{\phVar}_d - \dot{\phVar}) &, \ \phVar < \phVar_f \\
        -d_2\dot{\phVar} - k_2(\phVar - \phVar_f) &, \ \phVar \ge \phVar_f \\
	 \end{cases}
\end{equation}
with $d_1, d_2, k_2 > 0$, $d_2 \ge 2\sqrt{k_2}$, $\phVar(0) = 0$ and $\dot{\phVar}(0) = \dot{\phVar}_d = 1/T_f$, where $T_f$ is the desired motion duration. In the general case $\dot{\phVar}_d$ can also be time-varying.
Notice that \eqref{eq:GMP_can_sys} guarantees that the phase $\phVar$ convergences from $\phVar(0)=0$ to 
$\phVar \rightarrow 1$, $\dot{\phVar},\ddot{\phVar} \rightarrow 0$ as $t \rightarrow \infty$.

As shown in \cite{Antosidi_Rev_DMP}, the novel DMP formulation \eqref{eq:GMP_tf_sys}-\eqref{eq:Ks_vec} is mathematically equivalent to the classical DMP. In particular, (1)-(4) can be mathematically manipulated to be written in a form similar to the classical one:
\begin{align}
    &\dvect{z} = \vect{K}( \vect{g} - \vect{y} ) - \vect{D} \vect{z} + (\vect{g} - \vect{y}_0)\fterm(\phVar) \label{eq:GMP_equiv_zdot} \\
    &\dvect{y} = \vect{z} \label{eq:GMP_equiv_ydot} \\
    &\scalemath{0.96}{\fterm(\phVar) = \diag{1./(\vect{g}-\vect{y}_0)}\left(\ddvect{y}_{\phVar} + \vect{D}\dvect{y}_{\phVar} - \vect{K}(\vect{g} - \vect{y}_{\phVar})\right)} \label{eq:GMP_forc_term}
\end{align}
The system is globally asymptotically stable at $\vect{g}$, since for $t \rightarrow \infty$, $\fterm \rightarrow \vect{0}$, following from 
$\dot{\phVar}, \ddot{\phVar} \rightarrow 0$ implying $\dvect{y}_{\phVar}, \ddvect{y}_{\phVar} \rightarrow \vect{0}$ and $s \rightarrow 1 \Rightarrow \vect{y}_{\phVar} \rightarrow \vect{g}$. 

Spatial scaling is achieved through $\vect{K}_s$ from \eqref{eq:Ks_vec} and temporal scaling by changing $\dot{\phVar}_d$ in \eqref{eq:GMP_can_sys}. These are equivalent to changing $\{\vect{g}, \ \vect{y}_0\}$ or the temporal scaling $\tau$ according to $\dot{\tau} = d_1 (\tau_d - \tau)$ in the classical DMP \cite{Ijspeert2013}, resulting in the exact same trajectory with the same spatial and temporal scaling \cite{Antosidi_Rev_DMP}. 

Notice that to achieve temporal scaling, in contrast to the classical DMP formulation, there is no need to pre-multiply \eqref{eq:GMP_equiv_zdot} - \eqref{eq:GMP_equiv_ydot} by $\tau$ , as this is handled by $\dot{\phVar}, \ddot{\phVar}$ which are included in $\dvect{y}_{\phVar}, \ddvect{y}_{\phVar}$ in $\fterm(\phVar)$ (this can be better understood by the mathematical equivalence analysis in \cite{Antosidi_Rev_DMP}).

For encoding Cartesian orientation we can simply employ the quaternion logarithm, i.e. set $\vect{y} = \log(\vect{Q}*\bvect{Q}_0)$ where $\vect{Q} \in \mathbb{S}^3$ is the orientation as unit quaternion. We can then obtain the unit quaternion, rotational velocity and acceleration from $\vect{y}$ and its derivatives using the mappings \eqref{eq:quatExp}, \eqref{eq:omega_qlogDot} and \eqref{eq:omegaDot_qlogDDot} from Appendix A. 
Alternatively we can rewrite the transformation system as:
\begin{equation} \label{eq:GMP_omega_tf_sys}
    \dvect{\omega} = \dvect{\omega}_s - \vect{D}(\vect{\omega} - \vect{\omega}_s) -\vect{K}\log(\vect{Q}*\bvect{Q}_s)
\end{equation}
where $\vect{Q}_s$, $\vect{\omega}_s$ and $\dvect{\omega}_s$ are obtained from \eqref{eq:quatExp}, \eqref{eq:omega_qlogDot} and \eqref{eq:omegaDot_qlogDDot} from Appendix A using \eqref{eq:y_x_vec} and its analytically calculated derivatives.

\section{Proposed generalization} \label{sec:Proposed_generalization}

Here we present a novel spatial generalization scheme that does not suffer from the drawbacks of the classical DMP generalization, reported in \cite{Bio_DMP_2009} and further allows the on-line incorporation of dynamic via-points\footnote{The term "point" can refer either to joint positions, Cartesian position, Cartesian orientation or Cartesian pose.}, while preserving the shape of the demonstrated trajectory, ensuring a smooth trajectory generation even for abrupt target or via-point changes.
To this end, we remove the classical DMP scaling $\vect{K}_s$, and redefine \eqref{eq:y_x_vec} as 
\begin{equation}
    \vect{y}_{\phVar} = \vect{f}_{p}(\phVar)
\end{equation}
and achieve spatial generalization by optimizing online the DMP weights under the initial and final state constraints i.e. start at $t=0$ from $\vect{y}_0$  and reach the target $\vect{g}$\footnote{We assume that $\vect{g}$ can change between $0$ and $T_f$ and is constant afterwards.} at $t=T_f$ with  zero velocity/acceleration. 
In order to generate a motion profile similar to the demonstration we minimize the distance from the learned (from the demonstration) acceleration profile. This online optimization allows also the incorporation of via-points in the spatial generalization.
The above is translated to the following optimization problem that has to be solved at each time-step $i$:
\begin{subequations} \label{eq:opt_prob}
\begin{align}
    \text{min}_{\vect{W}} & \sum_{j=1}^{m} \norm{\partDerSq{\hvect{y}_{d,j}}{s}  - \vect{W}^T \partDerSq{\vphi_{j}}{s}}_2^2 \label{eq:cost_function} \\
    \text{s.t.} \quad & \vect{W}^T \vect{A}(0) = \vect{Y}_0 \label{eq:init_state_constr}\\
                      & \vect{W}^T \vect{A}(1) = \vect{G}_i \label{eq:final_state_constr}\\
                      & \vect{W}^T \vPhi_{v,i} = \vect{Y}_{v,i} \label{eq:vp_constr} \\
                      & \vect{W}^T \vect{C}_j = \vect{Y}_j \ , \ j=1...i  \label{eq:current_state_constr} 
\end{align}
\end{subequations}
where the learned acceleration profile is\footnote{Notice from \eqref{eq:GMP_can_sys} that, for nominal execution, $\ddot{\phVar} = 0$ and hence the acceleration, i.e. the $2$nd time derivative of \eqref{eq:y_x_vec}, is just $\partial^2 \hvect{y}_{d,j} / \partial \phVar^2$ scaled by the constant $1/T_f^2$.} $\partDerSq{\hvect{y}_{d,j}}{s} = \vect{W}_0^T \partDerSq{\vphi_{j}}{s}$ with $\vect{W}_0 = \text{argmin}_{\vect{W}} J(\vect{W}^T \vphi(s), \vect{y}_d)$.
The matrices in the first two equality constraints, \eqref{eq:init_state_constr}, \eqref{eq:final_state_constr}, with $\vect{A}(s) = [\vphi(s) \ \dot{\vphi}(s) \ \ddot{\vphi}(s)]$, enforce the initial state $\vect{Y}_0 = [\vect{y}_0 \ \vect{0}_{n \times 1} \ \vect{0}_{n \times 1} ]$ and final state $\vect{G}_i = [\vect{g}_i \ \vect{0}_{n \times 1} \ \vect{0}_{n \times 1} ]$ conditions. 
Via-point constraints, which can be used to locally modify the DMP trajectory are defined in \eqref{eq:vp_constr}, where $\vect{Y}_{v,i} = [\vect{y}_{v,1} \ ... \ \vect{y}_{v,L_i}]$ contains the via-points, with $L_i$ their number at the current time-step $i$, and $\vPhi_{v,i} = [\vphi(s_{v,1}) \ ... \ \vphi(s_{v,L_i})]$.\footnote{If the phase $s_{v,l}$, $l=1...L_i$, is not provided explicitly, a reasonable heuristic is to calculate it as $s_v = \text{argmin}_{s_k}\norm{\vect{y}_s(s_k) - \vect{y}_v}$, where $s_k$ in uniformly sampled in $(s, 1]$. Empirically, taking approximately $80$ samples is sufficient, as the minimization need not be exact.}
The constraints in \eqref{eq:current_state_constr} with $\vect{C}_j = \vect{C}(s_j) = [\vphi(s_j) \ \dot{\vphi}(s_j) \ \ddot{\vphi}(s_{j-1})]$ and $\vect{Y}_j = [\vect{y}_j \ \dvect{y}_j \ \ddvect{y}_{j-1}]$ encode the current and all previous state constraints up to timestep $i$.
The default option is to set $\vect{Y}_j = \vect{W}_{i-1}^T \vect{C}_j$. An alternative option is to set $\vect{Y}_j$ to the actual robot's state at step $j$ if it is also desirable to adapt online the DMP to changes of the robot's state induced by external signals, like measured forces.
Regardless of the choice of $\vect{Y}_j$, the current state constraint for $j=i$, ensures that a smooth trajectory is generated by \eqref{eq:GMP_tf_sys} even in the presence of abrupt target changes.
The previous state constraints $\vect{Y}_{1:i-1}$ guarantee that if the DMP is run in reverse as in \cite{Antosidi_Rev_DMP}, using the final adapted weights from the forward execution, the same trajectory (in reverse) will be executed. This could be expedient for safely retracting after the forward execution of a task, especially if the task has geometric constraints (like an insertion) which will have to be respected in the retraction phase.

\subsection{Online DMP weights adaptation} \label{subsec:online_DMP_weights_adapt}

\noindent Solving the problem in \eqref{eq:opt_prob} on-line at each control cycle would be too slow and hence prohibitive for real-time usage.
Instead we can employ the following recursive 2-step update at each time-step $i>0$, which has computational complexity O($K^2$) and can be carried out in real-time:
\\\textbf{step 1}: Remove the effect of the previous target $\vect{g}_{i-1}$ and via-points:
\begin{subequations} \label{eq:downdate}
\begin{align}
    \hvect{W}_{i} &= \vect{W}_{i-1} + \hvect{K}_i(\hvect{Z}_i - \vect{W}_{i-1}^T\hvect{H}_i)^T \label{eq:w_downdate} \\
    \hvect{P}_{i} &= \vect{P}_{i-1} - \hvect{K}_i\hvect{H}_i^T\vect{P}_{i-1} \label{eq:P_downdate} \\
    \hvect{K}_i &= \vect{P}_{i-1}\hvect{H}_i(\hvect{R}_i + \hvect{H}_i^T\vect{P}_{i-1}\hvect{H}_i)^{-1} \label{eq:Ki_downdate}
\end{align}
\end{subequations}
where
\begin{equation} \label{eq:step1_H_Z_R}
    \begin{aligned}
        \hvect{Z}_i &= [\vect{G}_{i-1} \ \vect{Y}_{v,i}^-], \ \hvect{H}_i = [\vect{A}(1) \ \vect{\Phi}_{v,i}^-], \\ 
        \hvect{R}_i &= -\epsilon \vect{I}_{3+b_i^-}
    \end{aligned}
\end{equation}
with $\epsilon \approx 0^+$
and $\vect{Y}_{v,i}^- \in \mathbb{R}^{n \times b_i^-}$ containing in columns the via-points that exist in the columns of $\vect{Y}_{v,i-1}$ and not in $\vect{Y}_{v,i}$ (i.e. are present at time-step $i-1$ and not at $i$)\footnote{Such is the case for instance if they are defined relative to a varying target or an object, whose position has changed or has been completely removed.} and $\vect{\Phi}_{v,i}^- = [\vphi(s_k)]_{k=1:b_i^-}$ with $s_k$ being the phase corresponding to each via-point in $\vect{Y}_{v,i}^-$.
\\\textbf{step 2}: Update to new target and via-points:
\begin{subequations} \label{eq:update}
\begin{align}
    \vect{W}_{i} &= \hvect{W}_i + \vect{K}_i(\vect{Z}_i - \hvect{W}_{i}^T\vect{H}_i)^T \label{eq:w_update} \\
    \vect{P}_{i} &= \hvect{P}_{i} - \vect{K}_i\vect{H}_i^T\hvect{P}_{i} \label{eq:P_update} \\
    \vect{K}_i &= \hvect{P}_{i}\vect{H}_i(\vect{R}_i + \vect{H}_i^T\hvect{P}_{i}\vect{H}_i)^{-1} \label{eq:K_i}
\end{align}
\end{subequations}
where
\begin{equation} \label{eq:step2_H_Z_R}
    \begin{aligned}
        \vect{Z}_i &= [\vect{Y}_i \ \vect{G}_i \ \vect{Y}_{v,i}^+], \ \vect{H}_i = [\vect{C}_i \ \vect{A}(1) \ \vect{\Phi}_{v,i}^+], \\ 
        \vect{R}_i &= \epsilon \ \vect{I}_{6 + b_i^+}
    \end{aligned}
\end{equation}
with $\vect{Y}_{v,i}^+ \in \mathbb{R}^{n \times b_i^+}$ containing in columns the via-points that are present at time-step $i$ and not at $i-1$ and $\vect{\Phi}_{v,i}^+ = [\vphi(s_k)]_{k=1:b_i^+}$ with $s_k$ being the phase corresponding to each via-point in $\vect{Y}_{v,i}^+$.

For initialization at step $i=0$ we set $\vect{Z}_0 = \begin{bmatrix} \vect{Y}_0 & \vect{G}_0 \end{bmatrix}$, $\vect{H}_0 = \begin{bmatrix} \vect{A}(0) & \vect{A}(1)\end{bmatrix}$, $\vect{R}_0 = \epsilon \vect{I}_6$, $\hvect{W}_0 = \vect{W}_0$ and $\hvect{P}_0 = \vect{P}_0$ where 
\begin{align*}
    \vect{W}_0 &= \scalemath{0.94}{\text{argmin}_{\vect{W}} J(\vect{W}^T \vphi(x), \vect{y}_d)} \\
    \vect{P}_0 &= \scalemath{0.9}{\left(\sum_{j=1}^m \partDerSq{\vphi_{j}}{s} \partDerSq{\vphi_{j}}{s}^T \right)^{-1}}
\end{align*}

\begin{remark} \label{remark:W_0}
    At the beginning of each execution, the DMP weights are initialized to the values that have been calculated offline once during the DMP training from the demonstration.
    Alternatively, if at a previous execution the DMP was modified, e.g. from via-points or the robot's actual state,
    resulting to some final weights $\vect{W}_f$, and it is desirable to retain this adjusted pattern at subsequent executions, one could set $\vect{W}_0 = \vect{W}_f$.
\end{remark}

\begin{remark} \label{remark:const_g}
    When $\vect{g}$ is constant and/or we don't wish to adapt the DMP trajectory to via-points or changes of the actual robot's state, we can drop the constraints in \eqref{eq:vp_constr}, \eqref{eq:current_state_constr}. It suffices to carry the above update only once at the beginning with $\vect{Z}_0$, $\vect{H}_0$ and $\vect{R}_0$ to obtain $\vect{W}_1$ and use $\vect{W}_i = \vect{W}_1$ for $i>1$.
\end{remark}

\begin{remark} \label{remark:choice_of_epsilon}
    In practice, due to finite numerical precision, setting $\epsilon$ in $\hvect{R}_i$, $\vect{R}_i$ very close to zero can make the matrix being inverted in \eqref{eq:Ki_downdate} and \eqref{eq:K_i} ill-conditioned. We have found that in practice choosing $\epsilon$ approximately in the range $(10^{-10}, \ 10^{-6})$ does not create any numerical issues and the constraints are satisfied within an error tolerance of the order of $10^{-4}$ or even less depending on how small $\epsilon$ is.
\end{remark}

\begin{remark} \label{remark:infeasibility_epsilon}
    Despite using the same value $\epsilon$ in the above analysis for simplicity, it is easy to verify that different values of $\epsilon$ can be chosen for each type of constraint. 
    The selected $\epsilon$ value can facilitate finding a solution in the optimization problem \eqref{eq:opt_prob} in case constraints cannot be strictly satisfied  (i.e. infinite accuracy). Then, 
    the value of $\epsilon$ essentially relaxes equality constraints, by  penalizing the cost function (see \theoref{theo:LS_equivalence} from Appendix B). Choosing different values of $\epsilon$ for each type of constraint, allows to prioritize which of these constraints should be satisfied with greater accuracy.
    Moreover, when adapting to the actual robot's state, which may be perturbed by noisy external signals, higher values of $\epsilon$ can be used to suppress that noise.
\end{remark}

\begin{remark} \label{remark:obj_cost}
    The proposed recursive algorithm for updating the DMP weights to dynamically changing targets and/or via-points while ensuring a smooth trajectory generation can also be employed with other objective functions as well.
    For instance, it can be easily verified that optimizing the position (instead of the acceleration) results in the same equations with $\vect{P}_0 = \scalemath{0.9}{(\sum_{j=1}^m \vphi_{j}\vphi_{j}^T )^{-1}}$. More general objective functions can also be specified, like the one from \cite{DMP_via_opt_2015}, which translates to defining $\vect{P}_0 = \scalemath{0.9}{\sum_{j=1}^m (\vphi_{j} \vect{M} \vphi_{j}^T )^{-1}}$, where $\vect{M}$ is a metric matrix that is learned from multiple demonstrations.
\end{remark}

\update{
\begin{remark} \label{remark:opt_with_other_scaling}
    The proposed recursive algorithm to include dynamic via-points can also be applied with the spatial scaling from \cite{Ijspeert2013}, \cite{Koutras_DMP_scaling} or \cite{Mollifier_bio_DMP_2021}. 
    In particular, one can use  $\vect{y}_{\phVar}$ as defined in \eqref{eq:y_x_vec} with the corresponding scaling matrix $\vect{K}_s$, resulting from \cite{Ijspeert2013}, \cite{Koutras_DMP_scaling} or \cite{Mollifier_bio_DMP_2021}, and employ the same recursive updates \eqref{eq:downdate}, \eqref{eq:update} by replacing $\vect{W}$ with $\vect{K}_s \vect{W}$ and transforming each position $\vect{y}$ (be it the current position, target, via-point etc.) according to $\vect{y} := \vect{y} - \vect{y}_0 + \vect{K}_s \hvect{y}_{d,0}$.
    This can also be combined with a different objective cost as explained in Remark \ref{remark:obj_cost}.
\end{remark}
}

\subsection{Derivation/Proof of the weights adaptation} \label{subsec:prob_proof}

\noindent In the following, we prove that \eqref{eq:downdate}, \eqref{eq:update} solve the initial optimization problem given in \eqref{eq:opt_prob}. To this end, we will utilize the theorems provided in Appendix B.
Problem \eqref{eq:opt_prob} can be written compactly as:
\begin{align}
    \text{min}_{\vect{W}} & \text{tr}\{ (\ddhvect{Y}_{d} - \vect{W}^T \ddvect{\Phi})^T (\ddhvect{Y}_{d} - \vect{W}^T \ddvect{\Phi}) \} \label{eq:opt_prob_compact} \\
    \text{s.t.} \quad & \vect{W}^T \bvect{H}_i = \bvect{Z}_i \nonumber
\end{align}
where
$\ddhvect{Y}_{d} = [\partDerSq{\hvect{y}_{d,1}}{s} \ ... \ \partDerSq{\hvect{y}_{d,m}}{s}] \in \mathbb{R}^{n \times m}$, $\ddvect{\Phi} = [\partDerSq{\vphi_{1}}{s} \ ... \ \partDerSq{\vphi_{m}}{s}] \in \mathbb{R}^{K \times m}$,
$\bvect{Z}_i = \begin{bmatrix} \vect{Y}_0 & \vect{G}_i & \vect{Y}_{1:i} & \vect{Y}_{v,i} \end{bmatrix} \in \mathbb{R}^{n \times m_1}$ and $\bvect{H}_i = \begin{bmatrix} \vect{A}(0) & \vect{A}(1) & \vect{C}_{1:i} & \vPhi_{v,i} \end{bmatrix} \in \mathbb{R}^{K \times m_1}$ with $m_1 = 6 + 3i + L_i$.
Based on \theoref{theo:LS_eq_constr} problem \eqref{eq:opt_prob_compact} is equivalent to:
\begin{equation} \label{eq:opt_prob_compact_equiv}
\begin{aligned} 
    &\text{min}_{\vect{W}} f_i(\vect{W}) \triangleq \text{tr} \{ (\ddhvect{Y}_{d} - \vect{W}^T \ddvect{\Phi} )^T (\ddhvect{Y}_{d} - \vect{W}^T \ddvect{\Phi} )\}  \\
    &\ \qquad \ + \text{tr} \{(\bvect{Z}_i - \vect{W}^T \bvect{H}_i)^T \bvect{R}_i^{-1} (\bvect{Z}_i - \vect{W}^T \bvect{H}_i) \} 
\end{aligned}
\end{equation}
for $\bvect{R}_i = \epsilon \ \vect{I}_{m_1}, \ \epsilon \approx 0^+$.
Given now the solution at timestep $i-1$ for problem \eqref{eq:opt_prob_compact_equiv} we want to find the solution for timestep $i$, for which the cost function can be written as:
\begin{align}
    f_i = \hat{f}_{i-1} + \text{tr} \{(\vect{Z}_i - \vect{W}^T \vect{H}_i)^T \vect{R}_i^{-1} (\vect{Z}_i - \vect{W}^T \vect{H}_i) \} \nonumber
\end{align}
where $\hat{f}_{i-1} = f_{i-1} - \text{tr} \{(\hvect{Z}_i - \vect{W}^T \hvect{H}_i)^T \hvect{R}_i^{-1} (\hvect{Z}_i - \vect{W}^T \hvect{H}_i) \}$.
Invoking \theoref{theo:LS_downdate} the solution to $\hat{f}_{i-1}$ is given by \eqref{eq:downdate}. Given now $\hvect{W}_i, \hvect{P}_i$ for $\hat{f}_{i-1}$ and using \theoref{theo:LS_equivalence} the solution to $f_i$ is equivalent to the solution of $\trace{(\vect{W}-\hvect{W}_i)^T\hvect{P}_i (\vect{W}-\hvect{W}_i) } + \text{tr} \{(\vect{Z}_i - \vect{W}^T \vect{H}_i)^T \vect{R}_i^{-1} (\vect{Z}_i - \vect{W}^T \vect{H}_i) \}$, which based on \theoref{theo:LS_update} is given by \eqref{eq:update}.

\section{Simulations} \label{sec:Simulations}

\noindent In the following we compare the DMP produced with the proposed spatial generalization, henceforth referred as DMP$^{++}$, with the classical DMP, as well as other SoA variations that have been proposed in the literature.
We further demonstrate and comment on the effects of adaptation to dynamic via-points and also including kinematic limits (inequality constraints) using the proposed generalization with \cite{Antosidi_DMP_constr_2022}.  
Unless stated otherwise, in all cases, based on \remarkref{remark:infeasibility_epsilon} and in the spirit that higher priority should be placed on the initial and final position constraints followed by the via-point constraints and then the previous state constraints, we chose $\epsilon$ $10^{-9}$ for the initial and final position, $10^{-7}$ for the initial and final velocity and acceleration constraints, $10^{-7}$ for via-points and $10^{-6}$, $10^{-6}$, $10^{-4}$, for the previous state constraint. Notice that violation of the previous state constraint generates a discontinuous acceleration in the DMP \eqref{eq:GMP_tf_sys}, which will be larger the bigger the error in the equality constraint is.
For via-points, the phase variable assigned at a via-point $\vect{y}_v$ is determined as $s_v = \text{argmin}_{s_k}\norm{\vect{y}_s(s_k) - \vect{y}_v}$, where $s_k$ in uniformly sampled in $(s, 1]$ taking $80$ points, where $s$ is the current value of the phase variable.

\subsection{Comparison with classical DMP}


\noindent We start off by examining the three basic cases where the classical DMP scaling is problematic \cite{Bio_DMP_2009}. For simplicity we consider an $1$ DoF demonstration  starting at position zero. Simulation results for each case are plotted in Fig. \ref{fig:classic_scale_drawbacks}, where the demo is plotted with green dashed line, the proposed DMP$^{++}$ with blue and the classical DMP with magenta dotted line.
In the first case (Fig. \ref{fig:classic_scale_drawbacks}, top subplot) the demonstrated target $g_d$ is quite close to   the initial position $y_{0,d}=0$. In this case, even a new target that is close to that of the demo, results in over-scaling for the classical DMP. In the figure, the resulting trajectory for the classical DMP is scaled by $400$ for visualization purposes (the position actually converges to $g$ but in the plot due to the scaling by $400$ it is $g/400 \approx 0$). In the second case (Fig. \ref{fig:classic_scale_drawbacks}, middle subplot) the demo initial and target positions are different but  
the new target is the same with the initial position
i.e. $g=y_0$. In this case, no motion is generated by the classical DMP. Finally, in the third case (Fig. \ref{fig:classic_scale_drawbacks}, bottom subplot), if the new target is chosen below the initial position, or more generally when $\text{sign}(g-y_0) = -\text{sign}(g_d-y_{0,d})$, the classical DMP produces a mirrored trajectory. Unlike the classical DMP, the proposed DMP$^{++}$ retains the shape of the demonstrated trajectory in all cases, without exhibiting any side-effects.

\begin{figure}[!ht]
    \centering
    \includegraphics[scale=0.57]{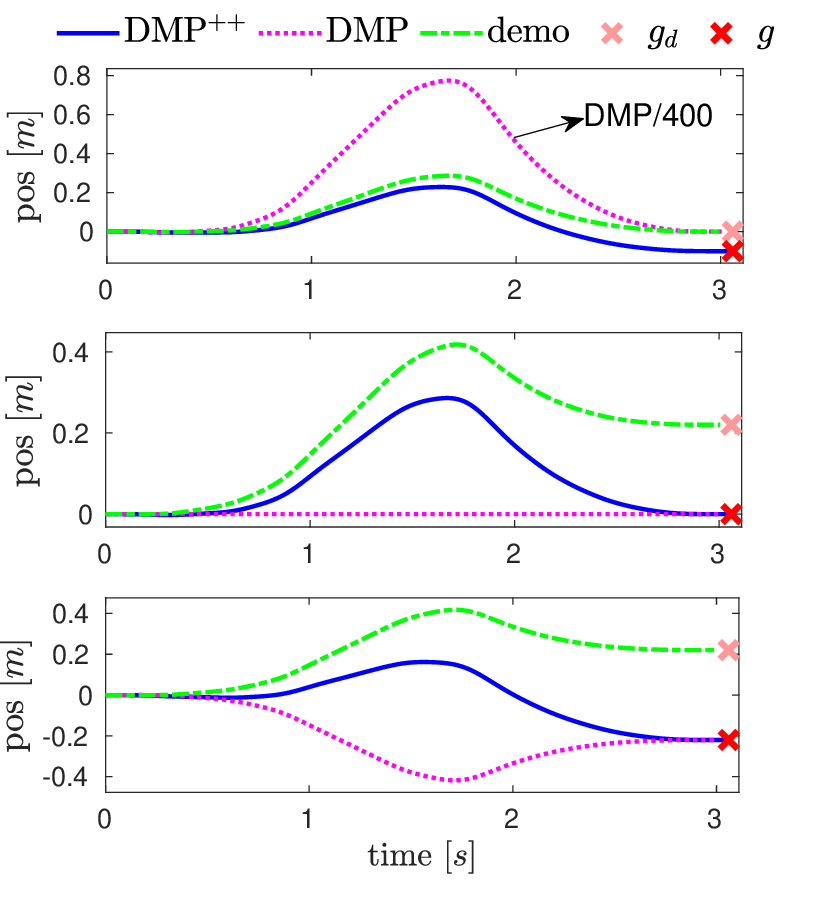}
    \caption{Comparison with classical DMP generalization. Top: The demo start and goal positions are close. Middle: New goal set at the start position. Bottom: Mirroring.}
    \label{fig:classic_scale_drawbacks}
\end{figure}

\begin{figure*}[!ht]
    \centering
    \includegraphics[scale=0.42]{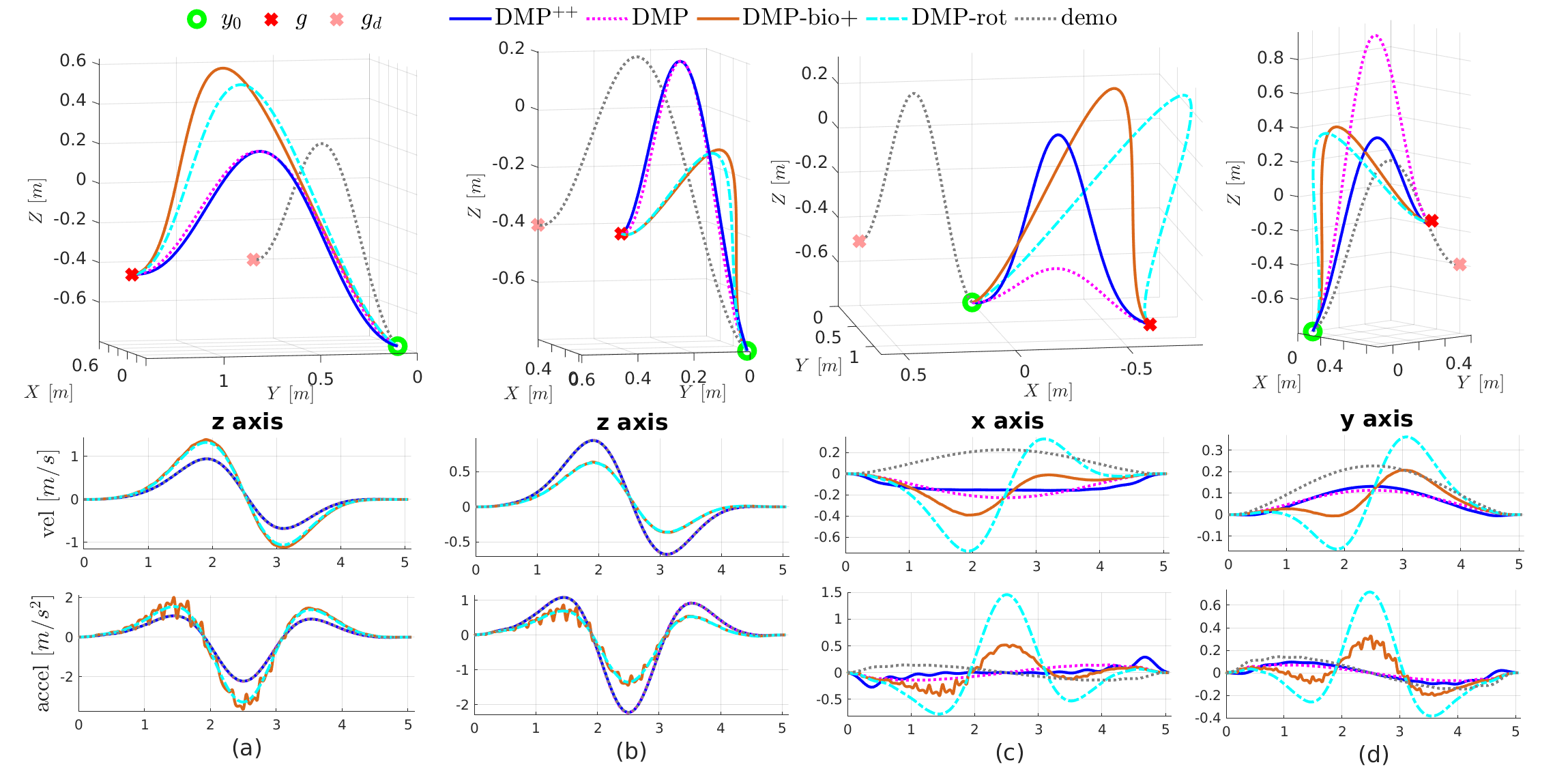}
    \caption{Spatial generalization comparison between the proposed DMP$^{++}$ and other DMP variants, for displacements of the new target $\vect{g}$ different from the demo target $\vect{g}_d$. In each column the top subplot depicts the $3D$ paths and the two bottom subplots the corresponding velocities and accelerations for an indicative axis. (a)/(b): Same $z$ and bigger/smaller displacement on the $xy$ plane from $y_0$. (c)/(d) Displacement in all $3$ axis, with lower/higher $z$.}
    \label{fig:compare_scalings}
\end{figure*}

\subsection{Comparison with other DMP variants}

\noindent We further compare the proposed generalization against other SoA DMP variants that propose a different generalization. In particular, we compare against \cite{Koutras_DMP_scaling} (we will call it DMP-rot), which proposes a different spatial scaling that remedies the $3$ aforementioned drawbacks of classical DMP. We also compare against the bio-inspired DMP variant from \cite{Mollifier_bio_DMP_2021} (we will call it DMP-bio$^+$), which remedies the scaling problem of the bio-inspired DMP \cite{Bio_DMP_2009} for new target positions that are not close to the demo target position.

Notice that both \cite{Koutras_DMP_scaling}, that assumes the classical DMP formulation, and \cite{Mollifier_bio_DMP_2021}, that assumes the bio-inspired DMP formulation, use for spatial generalization the scaling matrix $\vect{K}_s = \vect{R} \frac{\norm{\vect{g} - \vect{y}_0}}{\norm{\vect{g}_d - \vect{y}_{d,0}}}$, where $\vect{R} \in SO(3)$ is such that $\frac{\vect{g} - \vect{y}_0}{\norm{\vect{g} - \vect{y}_0}} = \vect{R} \frac{\vect{g}_d - \vect{y}_{d,0}}{\norm{\vect{g}_d - \vect{y}_{d,0}}}$. It becomes obvious from such a $\vect{K}_s$, and in particular from $\frac{\norm{\vect{g} - \vect{y}_0}}{\norm{\vect{g}_d - \vect{y}_{d,0}}}$, that displacement in one axis can affect the scaling in all axes, which in turn can result in over-amplification of the amplitude of an axis even if there is no displacement there. This effect is more pronounced for trajectories where in a specific axis the amplitude of the demo (i.e. the distance between the min and max position) is bigger than the distance between the original and final position. Such trajectories are typical in many real practical scenarios like packing.

To highlight these issues, we consider for the comparisons such a trajectory, with a higher amplitude than the distance between the initial and final demo position and examine the spatial generalization for $4$ different displacements of the new target $\vect{g}$ from $\vect{g}_d$.
The results are plotted in Fig. \ref{fig:compare_scalings}, where each column depicts the results of a target displacement, with the $3D$ paths of each DMP variant on top and the corresponding velocities and acceleration of an indicative axis at the bottom. In the first case (Fig. \ref{fig:compare_scalings}-(a)), we can observe that although both $\vect{g}$ and $\vect{g}_d$ have the same $z$, the bigger displacement of $\vect{g}$ from $\vect{y}_0$ on the $xy$ plane results in a large scaling for DMP-bio$^+$ and DMP-rot. This behaviour in practical pick and place tasks is arguably undesirable. For instance, when changing the $xy$ position of a box where an object is to be placed, but the height of the box is the same, there is no good reason why the height of the DMP trajectory should alter. Moreover, this magnification of the amplitude generates much larger velocities and accelerations, as can be seen in the bottom subplots of \ref{fig:compare_scalings}-(a)\footnote{Note that DMP-bio$^+$ produces a jerky acceleration which is owning to the use of the mollifier kernel which have finite support (instead of Gaussian kernels) \cite{Mollifier_bio_DMP_2021}.}. 
In contrast, DMP$^{++}$ does not alter the amplitude of the trajectory and exhibits a behaviour close to the classical DMP. Notice also that the velocities and acceleration of DMP$^{++}$ and DMP coincide with the demo, as one would expect given that the height of the target is unaltered. Analogous conclusions can be drawn in \ref{fig:compare_scalings}-(b), where the height of the target is the same, but this time the displacement on the $xy$ plane is smaller. Finally, in \ref{fig:compare_scalings}-(c),(d) the target is displaced in all axes, with lower or higher $z$ respectively for $\vect{g}$. Here, the scaling drawbacks of the classical DMP are clearly manifested. Notice also, that although DMP-rot and DMP-bio$^+$ do not generate large scalings and the resulting path retains the demo shape, it is nevertheless rotated in such a way that such a behaviour would not be very natural in practical tasks. (e.g. in a pick and place).

\subsection{Generalization to dynamic via-points}

Here we test the spatial generalization in the presence of dynamic via-points. We consider a single $2$D demonstration for placing an object inside a box. During simulation we consider a tighter box that is also higher than that of the demonstration, hence via-points are used to ensure the proper placement inside the box. A rectangular obstacle is also present in the scene, with via-points specified relative to it, so as to ensure its avoidance in a predictable manner. To emulate dynamic changes in the scene, which typical occur in real environments, the position of the box and the target are displaced at different time instances.

\begin{figure}[!t]
    \centering
    \includegraphics[scale=0.56]{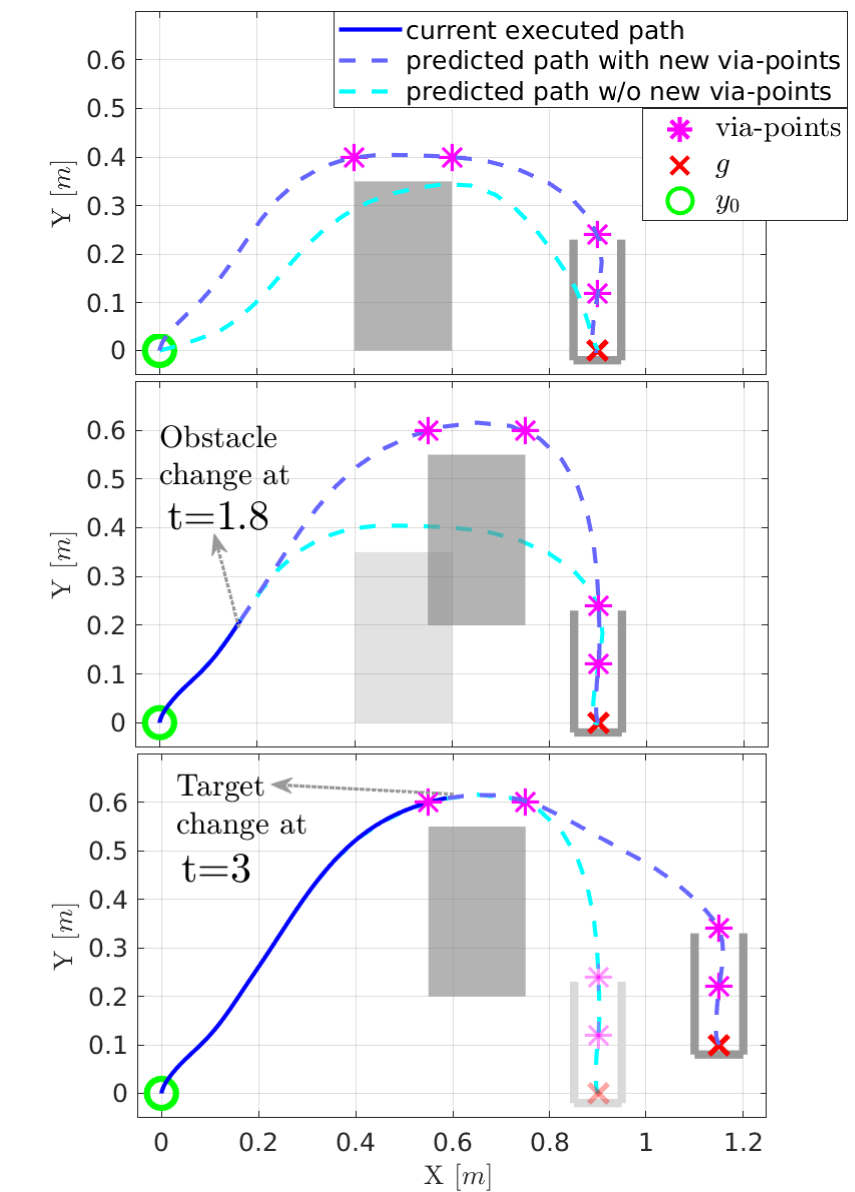}
    \caption{DMP$^{++}$ simulation with dynamic via-points. $1$st subplot: Initial predicted DMP path (i.e. the one that would be produced) with and w/o the via-points. $2$nd subplot: The obstacle is displaced and the current executed DMP path as well as the previous and updated predicted DMP path w/o and with the new via-points is shown. $3$rd subplot: The target is displaced.}
    \label{fig:vp_sim}
\end{figure}

\begin{figure}[!t]
    \centering
    \includegraphics[scale=0.4]{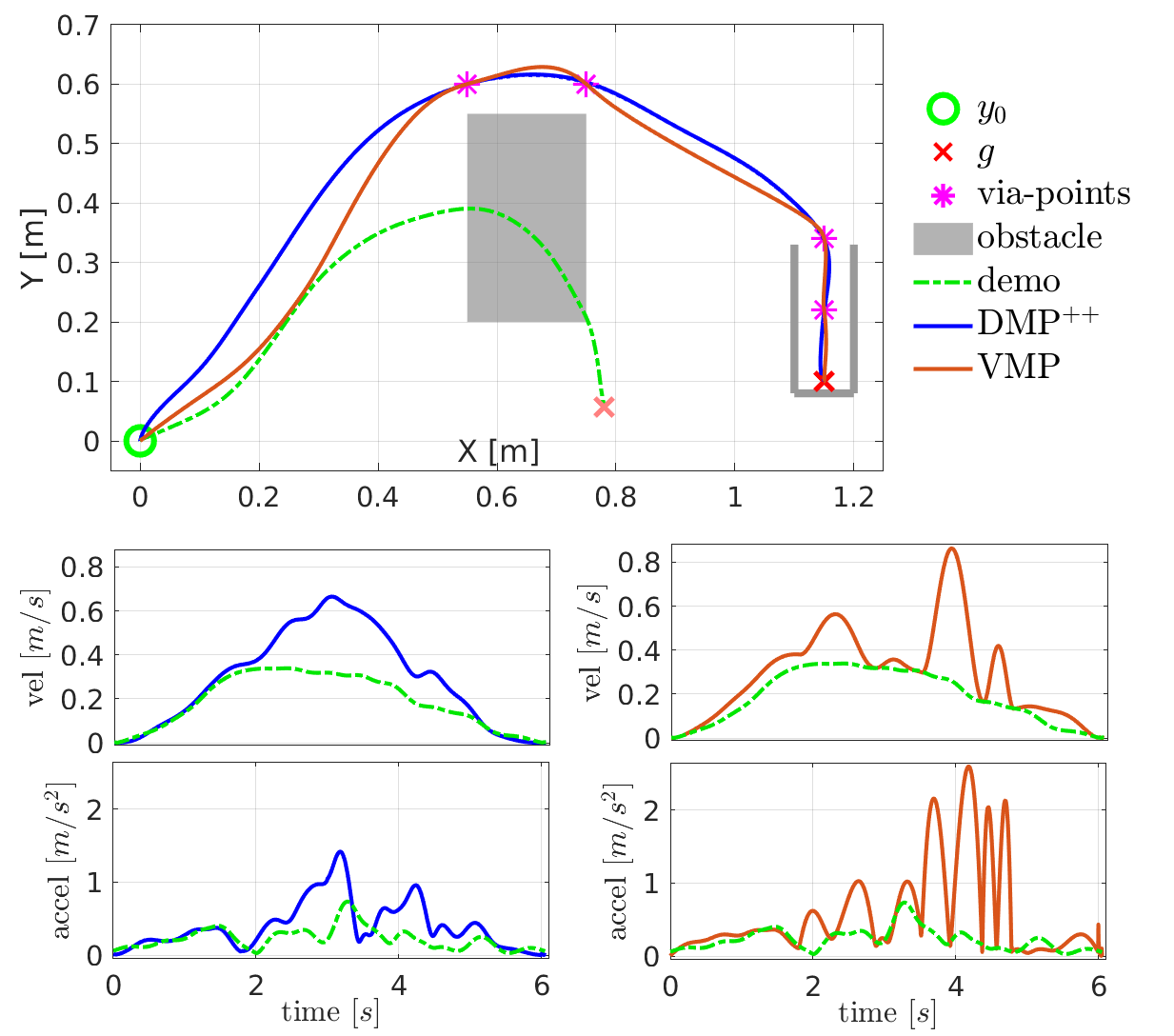}
    \caption{DMP$^{++}$ vs VMP. Top plot: $2$D paths. Bottom plots: velocity and acceleration L$2$ norms.}
    \label{fig:vp_comp}
\end{figure}

We compare DMP$^{++}$ with VMP \cite{VMP_2019} 
\footnote{Other methods like \cite{DMP_plus_2016} and \cite{Mollifier_bio_DMP_2021} are able to adapt only to a new continuous trajectory segment and not via-points. Also, ProMP \cite{ProMPs_Paraschos_2018} and KMP \cite{KMP_2019} require multiple demonstrations and do not address the case of dynamic via-points.}.
VMP are tested in \cite{VMP_2019} only with static via-points, but can be adapted to dynamic via-points too. Specifically, a VMP is given by:
\begin{equation}
    \vect{y}(\phVar) = \vect{f}_p(\phVar) + \vect{A}^T \vphi_2(\phVar)
\end{equation}
where $\vect{f}_p(\phVar)$ is given by \eqref{eq:fp_vec}, $\vect{A} \in \mathbb{R}^{6 \times n}$ contains in each column the coefficients of a $5$th order polynomial for each DoF and $\vphi_2(\phVar) = [1 \ \phVar \ \phVar^2 \ \phVar^3 \ \phVar^4 \ \phVar^5]^T$. The weights $\vect{W}$ in $\vect{f}_p$ are fitted to the demo, while $\vect{A}$ is adapted according to the via-points, with the initial and target position treated as the first and final via-point. In particular, given the current phase value $s$, and the previous and next via-points $\vect{y}_v(\phVar_1)$, $\vect{y}_v(\phVar_2)$ with $\phVar_1 \le \phVar < \phVar_2$, $\vect{A}$ is determined by solving the system of equations:
\begin{align*}
    \begin{bmatrix}
    \vphi_2(\phVar_j)^T \\ \dvect{\phi}_2(\phVar_j)^T \\ \ddvect{\phi}_2(\phVar_j)^T
    \end{bmatrix}
    \vect{A}
    =
    \begin{bmatrix}
    (\vect{y}_v(\phVar_j) - \vect{f}_p(\phVar_j))^T \\ (\dvect{y}(\phVar_j) - \dvect{f}_p(\phVar_j))^T \\ (\ddvect{y}(\phVar_j) - \ddvect{f}_p(\phVar_j))^T
    \end{bmatrix} \ , \ j \in \{1, 2\}
\end{align*}
Similar to \cite{VMP_2019}, we set the velocity and acceleration at each via-point equal to that of the demonstration. Since via-points may alter on the fly, to ensure continuity we consider as the first via-point the current VMP point, i.e. we set $\phVar_1 = s$ and $\vect{y}_v(\phVar_1) = \vect{y}(\phVar)$ (except for the initial via-point, which has position $\vect{y}_0$). 
If the via-points are static, this approach generates the same trajectory as in \cite{VMP_2019}, since once $\vect{A}$ is adapted from a via-point $\vect{y}_{v}(\phVar_1)$ to a next via-point $\vect{y}_v(\phVar_2)$, then continuously updating $\vect{A}$ between $\vect{y}_{v}(\phVar) = \vect{y}(\phVar)$ and $\vect{y}_v(\phVar_2)$, with $\phVar \in (\phVar_1, \ \phVar_2)$ obviously yields the same $\vect{A}$. 

The simulated scenario and the results with DMP$^{++}$ are depicted in Fig. \ref{fig:vp_sim}. At the top subplot, the initial target and obstacle are shown, with the predicted DMP trajectory (i.e. the one that would be produced) without the via-points plotted with dashed cyan line, and after the adaptation to the via-points with light blue dashed line.
At $t=1.8$ sec, the obstacle is abruptly displaced (Fig. \ref{fig:vp_sim}, middle subplot), and the previous predicted DMP path and the new predicted DMP path, based on the updated via-points, are plotted again. The DMP path that has been executed so far is also shown with solid blue line. Finally, at $t=3$ sec, the target is abruptly displaced (Fig. \ref{fig:vp_sim}, bottom subplot).
In all simulation snapshots, comparing the cyan with the light blue dashed line, it's obvious that without updating the DMP to the new via-points, either a collision with the obstacle would occur or failure to reach the target due to bumping at the box's boundaries. 

We carried out the same simulation with VMP and plot the final results against the DMP$^{++}$ in Fig. \ref{fig:vp_comp}, where on the top subplot the $2$D paths are drawn and on the bottom subplots the velocity and acceleration norms against the demonstrated ones. It can be observed from the velocities and accelerations that VMP generates more abrupt and higher velocities and accelerations, that differ significantly compared to the demo. This shortcoming is due to the fact that the VMP adaptation 
essentially overrides the demo with the only criterion being the satisfaction of the two currently active via-point constraints. In contrast, DMP$^{++}$ performs a global optimization (see \eqref{eq:opt_prob}) that considers all via-points and adapts in a way as consistent as possible with the demo.

\subsection{Adding kinematic inequality constraints}

\noindent Here we compare
with the framework from \cite{Antosidi_DMP_constr_2022} which uses the DMP formulation \cite{Antosidi_Rev_DMP} with the classical DMP scaling to generate the generalized trajectory $\{\vect{y}_d, \dvect{y}_d, \ddvect{y}_d\}_i$ over a horizon of $N$ future time-steps which is then optimized to respect kinematic limits like position velocity and acceleration bounds as well as passing from via-points. We show that by modifying the framework from \cite{Antosidi_DMP_constr_2022} according to Fig. \ref{fig:dmp_star}, where the DMP with the classical scaling is replaced by DMP$^{++}$ we achieve more efficient generalization that accounts also for via-points.
Notice that via-points are included both in DMP$^{++}$  and in the optimization module of \cite{Antosidi_DMP_constr_2022} (Fig. \ref{fig:dmp_star} brown rectangle) via the respective equality constraints
so as to guarantee that the constrained within the kinematic inequality limits trajectory will also pass from the via-points.

\begin{figure}[!t]
    \centering
    \includegraphics[scale=0.135]{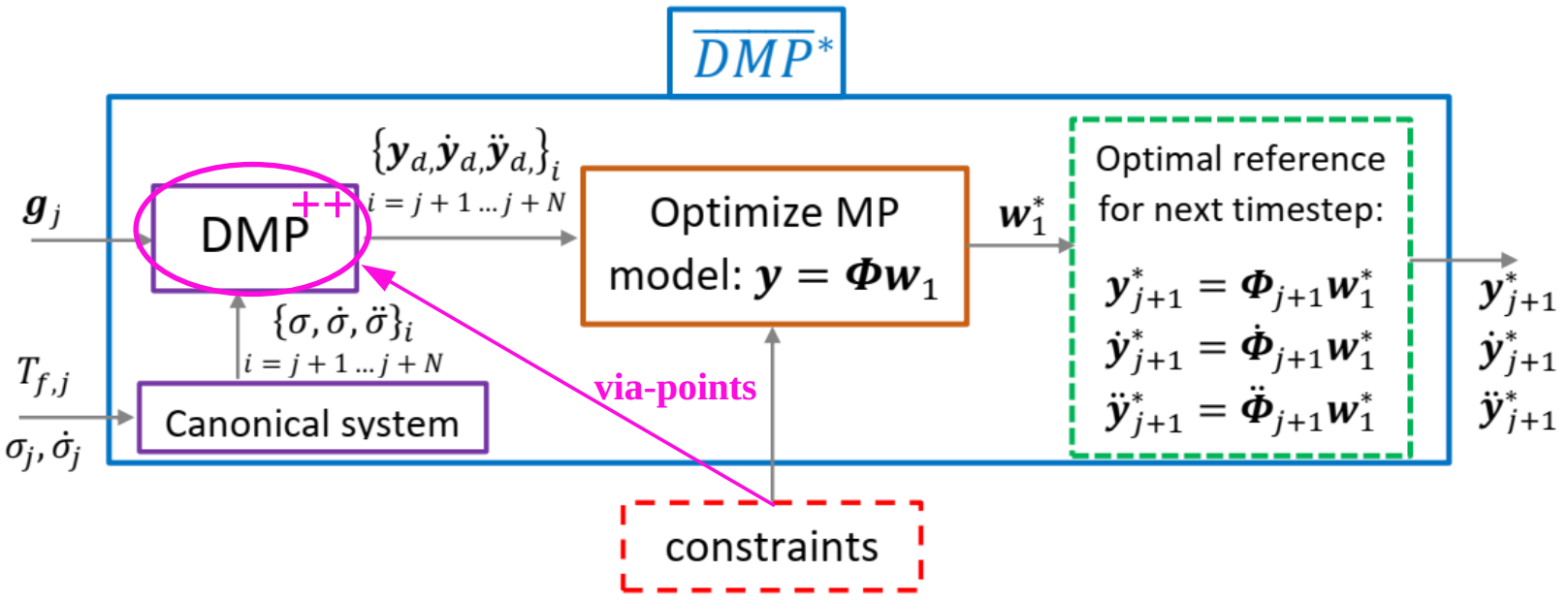}
    \caption{Combination of the proposed spatial generalization with \cite{Antosidi_DMP_constr_2022} to further impose kinematic inequality constraints. The image is taken from \cite{Antosidi_DMP_constr_2022}, where the differences/additions are highlighted with magenta. DMP$^{++}$ replaces the DMP with the classical spatial scaling. Via-points are also considered during the generation of the generalized (unconstrained) trajectory $\vect{y}_d$.}
    \label{fig:dmp_star}
\end{figure}

Simulation results of a $1$ DoF example with two via points and position, velocity and acceleration limits  are shown in Fig. \ref{fig:ineq_constr}, where the method proposed in \cite{Antosidi_DMP_constr_2022} is used to optimize the velocity profile with all relevant parameters chosen as in \cite{Antosidi_DMP_constr_2022}. 
To endow the optimizer with greater flexibility in finding feasible solutions we consider the kinematic limits as hard limits (grey dashed lines) and introduce the lower and upper soft limits (magenta dashed lines) within the hard limits. We want to preferably operate within the soft limits and only exceed them if feasibility would be inevitable otherwise \cite{Antosidi_DMP_constr_2022}.
It can be observed that the DMP$^{++}$ (blue line) retains the shape of the demonstration (green dash dotted line), even in the presence of kinematic limits and via-points (red asterisks). On the other hand, the DMP with the classical scaling (mustard dotted line) can induce large scalings that generate velocities and accelerations that violate considerably the limits, leading to saturation (see the $2$nd and $3$rd subplots of Fig. \ref{fig:ineq_constr}) and consequently the distortion of the demonstrated shape, as can be observed comparing the mustard with the green trajectory in the first subplot. This distortion would be even more adverse for cases in which the scaling is problematic like those presented in Fig. \ref{fig:classic_scale_drawbacks}, where a feasible solution may even not be found.

\begin{figure}[!t]
    \centering
    \includegraphics[scale=0.46]{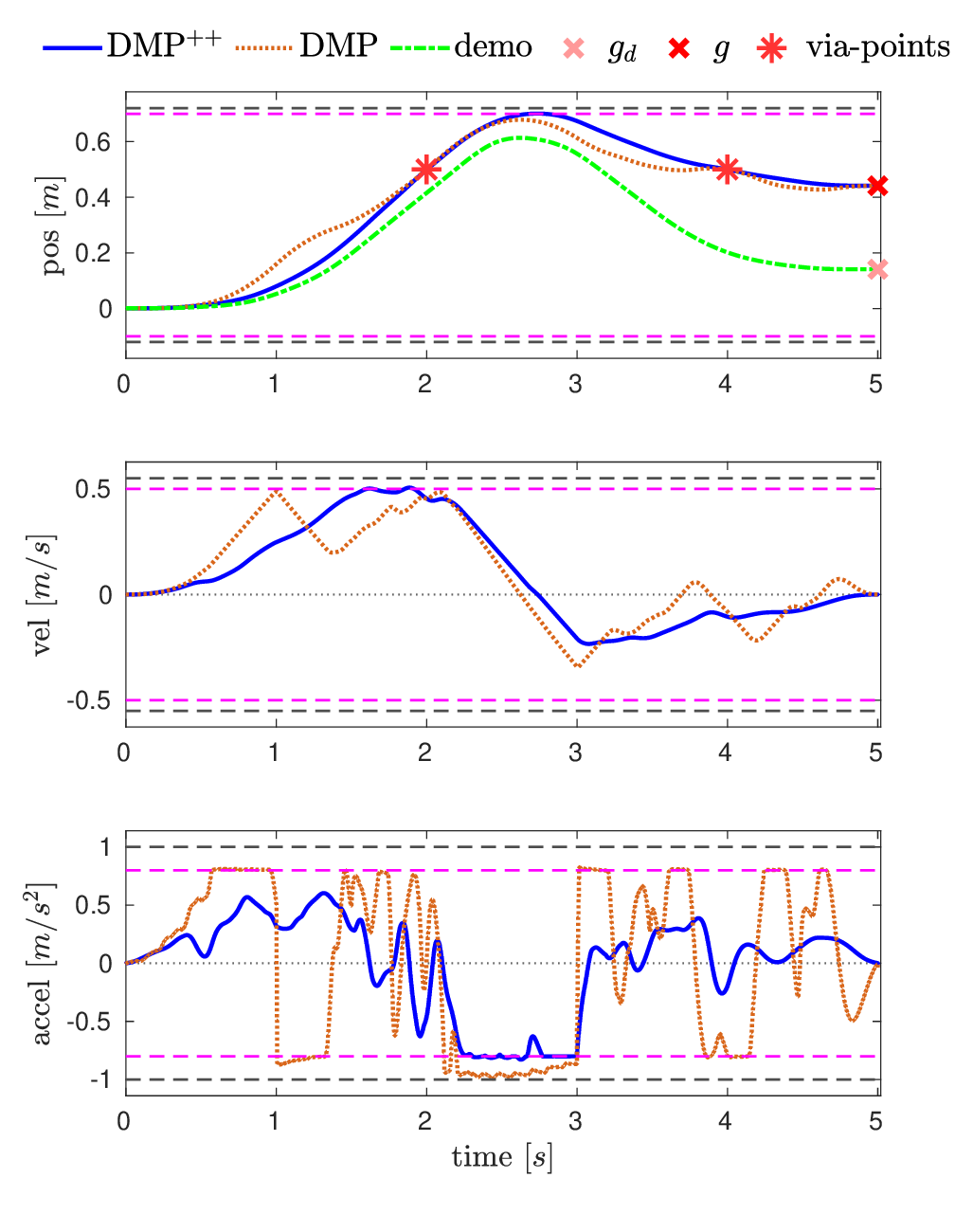}
    \caption{Comparison of the classical DMP spatial generalization against the proposed DMP$^{++}$ in combination with \cite{Antosidi_DMP_constr_2022} to enforce kinematic inequality constraints.}
    \label{fig:ineq_constr}
\end{figure}


\section{Experimental Validation} \label{sec:Experiments}

In this section, we further validate and demonstrate  experimentally the efficiency of DMP$^{++}$ in  practical pick and place tasks involved in packing scenarios. In particular we showcase our method's spatial generalization properties in a dynamic environment and incorporation of dynamic via-points.

\subsection{Packing scenario}

\noindent 
The robot has to pick carton products and place them inside boxes, where the size of the product and associated packing box can vary. A single demonstration is provided with a small carton ($5\times12\times5$ cm) and a short box ($10$ cm height) (Fig. \ref{fig:demo}). This demonstration is encoded in a DMP.
Then we use this DMP to execute three scenarios: 1) packing the last carton product used in the demo in a dynamically changed  box pose, 2) packing a larger carton product ($12\times12\times20$ cm) to its associated box ($28$ cm height) and 3) the latter scenario in the presence of dynamic obstacles. In all three scenarios, 
depending on the height of the carton and box, via-points are specified above the target pose to ensure a proper insertion of the carton in the box. Moreover, in all cases the box's target pose is altered on the fly by a human, thus the associated target via-points change also dynamically. In the $3$rd scenario, via-points are specified w.r.t. the obstacle to avoid it in a predictable way. As the obstacle is introduced dynamically, the associated via-points also change dynamically.


\subsection{Technical details}

To carry out these experiments a velocity controlled ur5e robot is used, with $2$ ms control cycle, which takes as reference the velocity produced by the DMP transformation system from \eqref{eq:GMP_tf_sys}, where $\vect{y}$ is either the Cartesian position or for orientation $\vect{y} = \log(\vect{Q}*\bvect{Q}_0)$. 
The DMP weights are updated online according to \eqref{eq:downdate}, \eqref{eq:update}, while $\epsilon$ is set as in the simulations. 
The current state constraint in the optimization is set equal to the last state generated by DMP$^{++}$. 
In all cases, the optimization at each control cycle was below $0.8$ ms,
which is well within the $2$ms control cycle of the robot.
For the canonical system from \eqref{eq:GMP_can_sys},
we use $\dot{\phVar}_d = 1/T_f$, with $T_f=10$ sec.

\begin{figure}[!htbp]
    \centering

    \begin{subfigure}[b]{\linewidth}
        \centering
        \includegraphics[width=\linewidth]{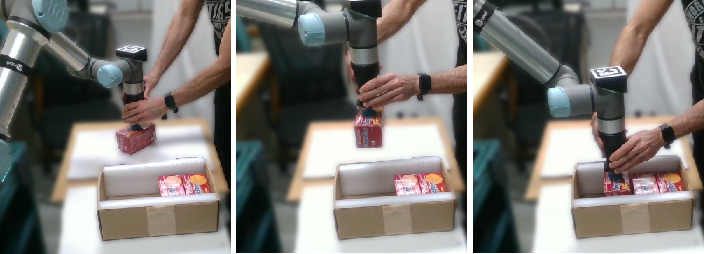}
        \caption{Demonstration: Placing a carton inside a box.}
        \label{fig:demo}
    \end{subfigure}%

    \begin{subfigure}[b]{\linewidth}
        \centering
        \includegraphics[width=\linewidth]{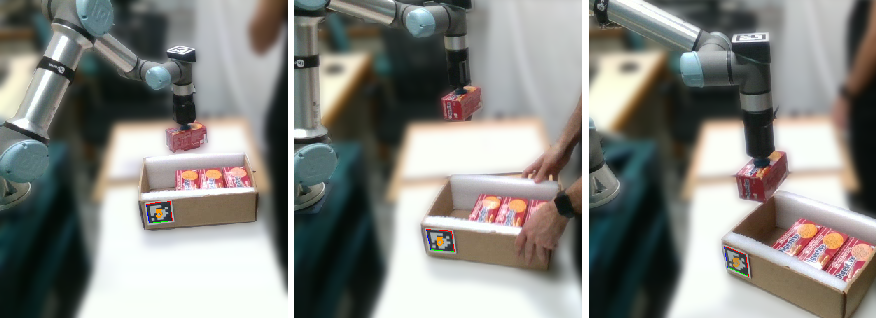}
        \caption{Scenario 1: small carton/box}
        \label{fig:c1_snaps}
    \end{subfigure}%

    \begin{subfigure}[b]{\linewidth}
        \centering
        \includegraphics[width=\linewidth]{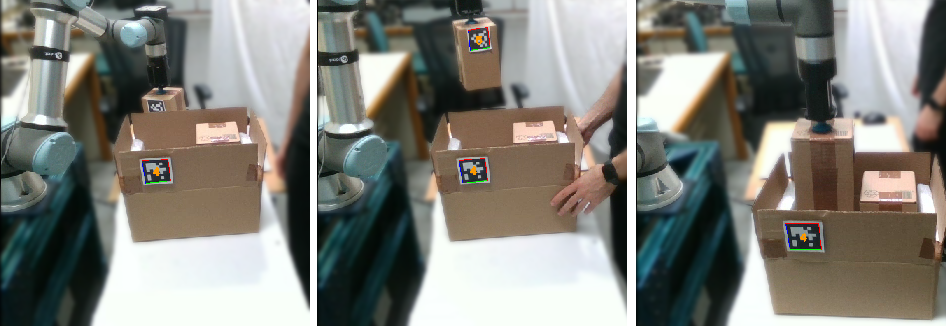}
        \caption{Scenario 2: large carton/box}
        \label{fig:c2_snaps}
    \end{subfigure}%

    \begin{subfigure}[b]{\linewidth}
        \centering
        \includegraphics[width=\linewidth]{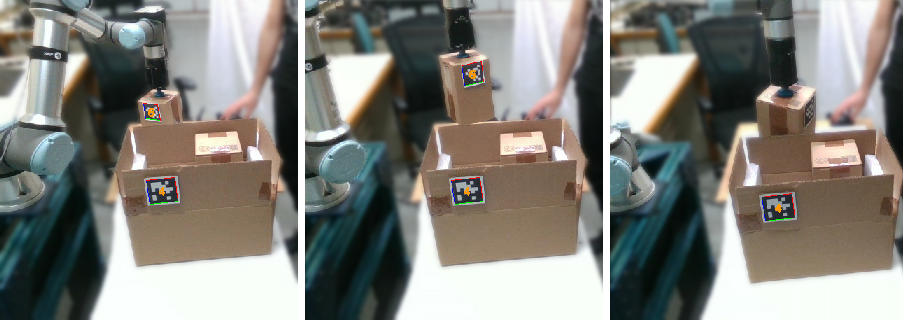}
        \caption{Scenario 2: large carton/box (NO via-points)}
        \label{fig:c0_snaps}
    \end{subfigure}%

    \begin{subfigure}[b]{\linewidth}
        \centering
        \includegraphics[width=\linewidth]{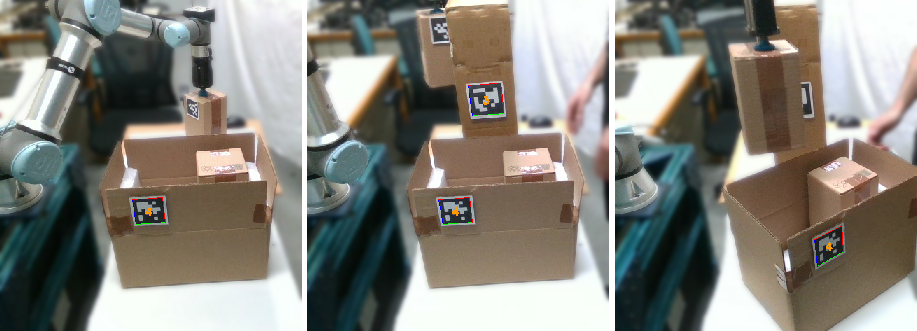}
        \caption{Scenario 3: large carton/box $+$ obstacle.}
        \label{fig:c3_snaps}
    \end{subfigure}%
    
    \caption{Demonstration and execution snapshots.}
    \label{fig:exp_snaps}
\end{figure}

The pose of all object's that are employed in the experiments are defined by apriltags which are tracked using a realsense2 camera at $30$ Hz. In general the target pose where the carton should be placed in the box can be provided by a perception system that determines the packing plan given the current box packing/occupancy state.
For simplicity, we predefine the target placing pose relative to the box's pose. For the short box we employ $2$ via-points placed every $5$ cm above the target pose and for the taller box $5$ via-points placed every $6$ cm above the target. These via-points have the same orientation as the target pose.
For the obstacle a higher level perception system could be employed for generating via-points to circumvent the obstacle.
Developing such a system is beyond the scope of this work, therefore we consider for simplicity two predefined via-points relative to the obstacle's center pose on the $xy$ plane. The $z$ position and orientation at each via-point is set equal to the DMP$^{++}$ reference at the corresponding phase variable instance. 
To determine whether these via-points should be incorporated in the DMP, we consider an augmented bounding box around the obstacle, equal to the size of the obstacle $+$ the half size of the carton. At each control cycle, we sample the DMP reference trajectory at $30$ future points and check if any of these points are inside the augmented box in order to use the obstacle's via-points in the DMP.

\subsection{Results}


For the first scenario (small carton, short box), snapshots of the execution are shown in Fig. \ref{fig:c1_snaps}, while the results are plotted in Fig. \ref{fig:c1_results}. In particular, in Fig. \ref{fig:c1_path}, the oriented path executed by the robot is shown, where also the initial and final box, target pose and associated via-points are visualized (the initial ones with opaque colors and the final ones with more vivid colors). The target change is also indicated by a light red line and we further plot the demonstrated path (cyan dotted line) for comparison. In Fig. \ref{fig:c1_profile} the translational and rotational velocity norm profile for the execution (blue line) and the demo (green line) are plotted, as well as  the target  displacement dynamics (magenta line). 
The demo velocity profile is scaled temporarily to the execution time duration, by multiplying the velocity with $T_{f,d} / T_f$, where $T_{f,d}=13.5$ sec is the demo duration. 
Moreover, the distance between the initial and target demo poses are $\norm{\vect{p}_{g,d} - \vect{p}_{0,d}} = 0.39$ m for the position and $\norm{\log(\vect{Q}_{g,d}*\vect{Q}_{0,d})} = 48$ degrees for the orientation, while during execution the corresponding initial displacements were larger with values $\norm{\vect{p}_{g}(0) - \vect{p}_{0}} = 0.45$ m and $\norm{{\log(\vect{Q}_{g}(0)*\vect{Q}_{0})}} = 67$ degrees. These larger displacements during execution cause the velocity profile to scale up initially (for $t \in [0 \ 3]$ sec) compared to the demo. As the target position is further displaced after $t=3$ sec, the velocity during execution further scales up compared to the demo. Nevertheless, the scaled execution norm velocity profile resembles the demonstrated one as expected. Some discrepancies are expected due to the target change and the continuity and via-point constraints. Observe also from the demo orientation velocity profile (Fig. \ref{fig:c1_profile} bottom) that the orientation settles before the final target position is reached. This is also the case in the execution, where however, due to the orientation displacement by a total of $37$ degrees (from $t=3$ until $t=4.2$ sec), the velocity profile is readjusted to reach the new target pose. 
Notice also that despite the target displacement which is a bit noisy and updated at a lower rate ($33$ ms compared to the $2$ ms control cycle), the generated velocity profile is smooth owning to the current state constraint in \eqref{eq:current_state_constr}.

\begin{figure}[!htbp]
    \centering
    \begin{subfigure}[b]{\linewidth}
        \centering
        \includegraphics[scale=0.4]{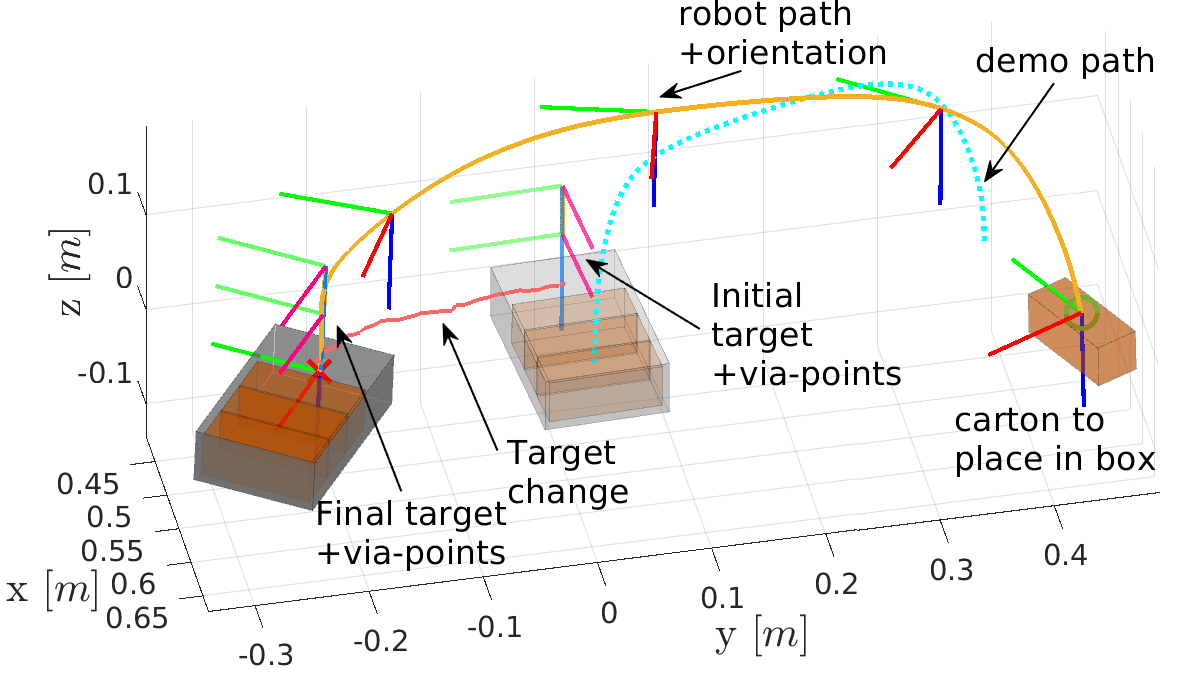}
        \caption{}
        \label{fig:c1_path}
    \end{subfigure}%

    \begin{subfigure}[b]{\linewidth}
        \centering
        \includegraphics[scale=0.52]{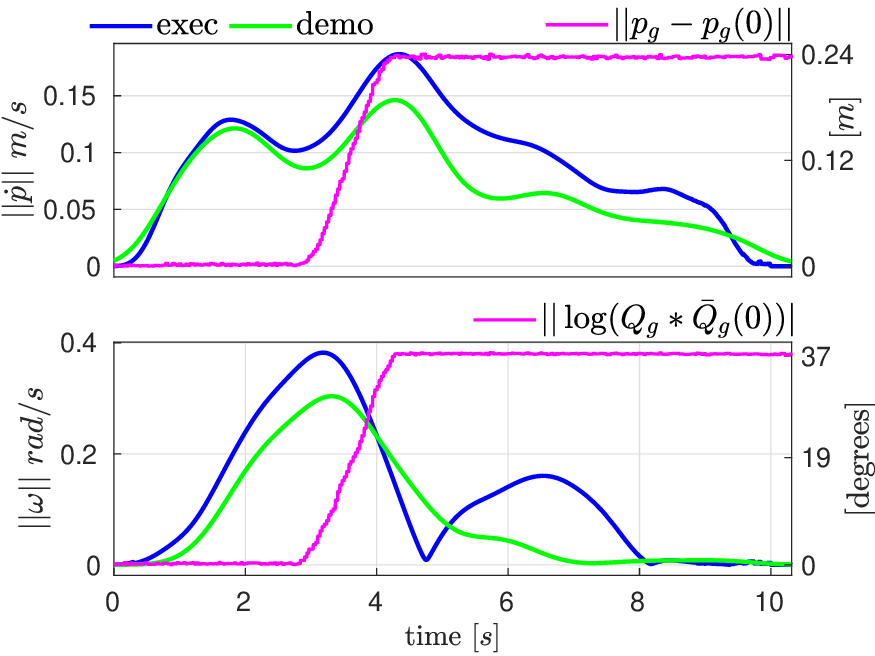}
        \caption{}
        \label{fig:c1_profile}
    \end{subfigure}%
    
    \caption{Experiment 1: Place small carton in short box. The box is displaced online. (a): Robot oriented Cartesian path and visualization of the carton, box and the associated target and via-points. (b): Velocity and target change.}
    \label{fig:c1_results}
\end{figure}


For the second scenario (large carton, tall box), snapshots of the execution are shown in Fig. \ref{fig:c2_snaps}, while the results are plotted in Fig. \ref{fig:c2_results}, with the Cartesian oriented path and the relevant via-points in Fig. \ref{fig:c2_path} and the translational and rotational velocity profile as well as the target change in Fig. \ref{fig:c2_profile}. Despite the larger shape of the carton and the box, the via-points contribute in the successful execution.
Notice that without the via-points, the execution would fail, as shown in the execution snapshots in Fig. \ref{fig:c0_snaps}. Training a new DMP with a new demonstration for the larger boxes could solve this issue. However this is time-consuming and impractical, as for every different object shape, a different DMP would be required. Instead, one can employ a single DMP, which already encapsulates the general desired motion shape-pattern, and adjust-scale it according to the manipulated object's size using via-points in the DMP spatial generalization. 
\update{Notice also that using two separate DMPs, one for reaching the box from the top and the second for inserting the carton would increase the system's complexity with an additional DMP. Moreover, as also discussed in the introduction, in more general and dynamic scenarios, using via-points is more flexible than keeping track of multiple DMPs to connect the via-points and also making additional modifications to ensure a smooth transition between the DMPs. In our case, passing through the via-points and the continuity and smoothness of the motion is handled automatically through the constraints in \eqref{eq:opt_prob}, which further makes no prior assumption regarding the number of via-points, while also accommodating the case where they change dynamically.}

\begin{figure}[!htbp]
    \centering
    
    \begin{subfigure}[b]{\linewidth}
        \centering
        \includegraphics[scale=0.47]{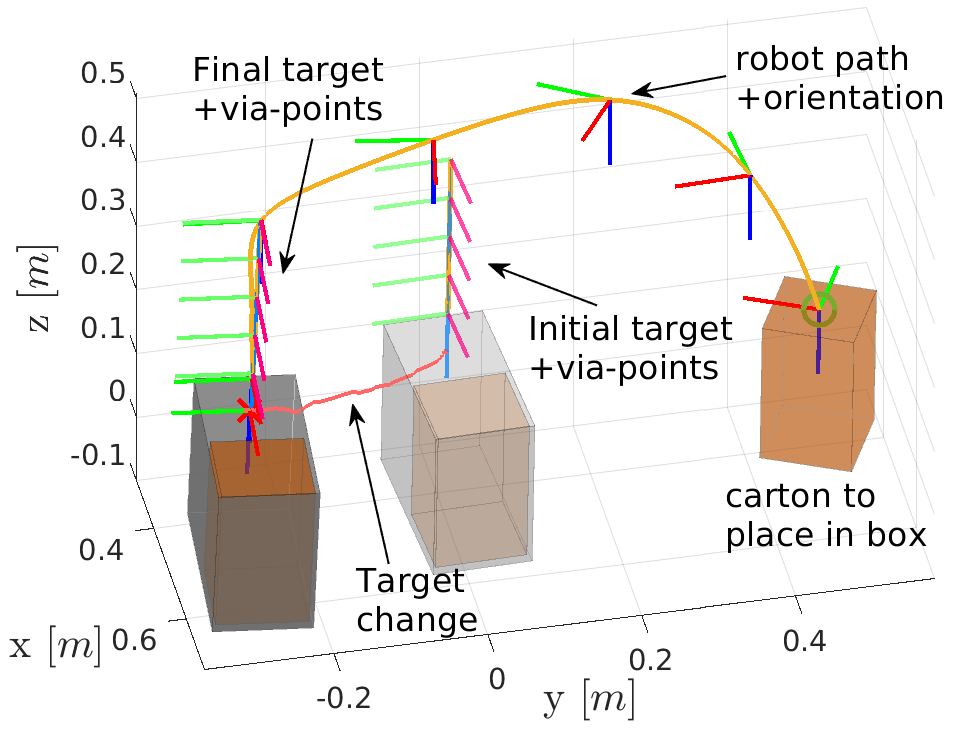}
        \caption{}
        \label{fig:c2_path}
    \end{subfigure}%

    \begin{subfigure}[b]{\linewidth}
        \centering
        \includegraphics[scale=0.52]{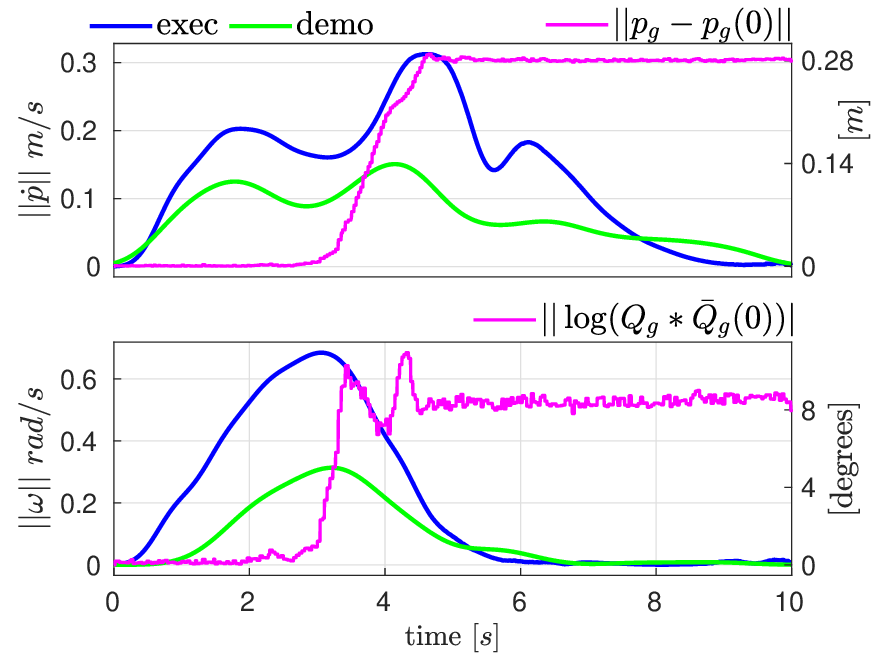}
        \caption{}
        \label{fig:c2_profile}
    \end{subfigure}%
    
    \caption{Experiment 2: Place large carton in tall box. The box is displaced online. (a): Robot oriented Cartesian path and visualization of the carton, box and the associated target and via-points. (b): Velocity and target change.}
    \label{fig:c2_results}
\end{figure}

Finally, in the last scenario shown in Fig. \ref{fig:c3_snaps}.
we see yet another case of incorporating dynamically via-points in the DMP generalization, in the context of obstacle avoidance. 
\update{Via-points can prove useful in such scenarios to steer the robot along a root that is more ergonomic and clutter-free. This can also be combined with the use of artificial potential functions for volumetric obstacle avoidance \cite{Ginesi2021_DMP_obst} to ensure that no collisions occur. Since this is not the main focus in this work and to keep things simple we define the via-points at a distance from the obstacle, considering the dimensions of the obstacle and the carton.
The results are plotted in Fig. \ref{fig:c3_results} where}
the final obstacle pose is plotted with magenta, along with its augmented bounds, shown in light red. The via-points introduced by the obstacle are depicted with cyan asterisks, from which the robot passes. Notice that the DMP trajectory can in general pass through the augmented obstacle bounds. They are not a forbidden region. Instead, they are used here as a rough estimate to detect when to include the obstacle via-points, from which if the DMP passes, it should avoid the obstacle.

\begin{figure}[!htbp]
    \centering
    
    \begin{subfigure}[b]{\linewidth}
        \centering
        \includegraphics[scale=0.5]{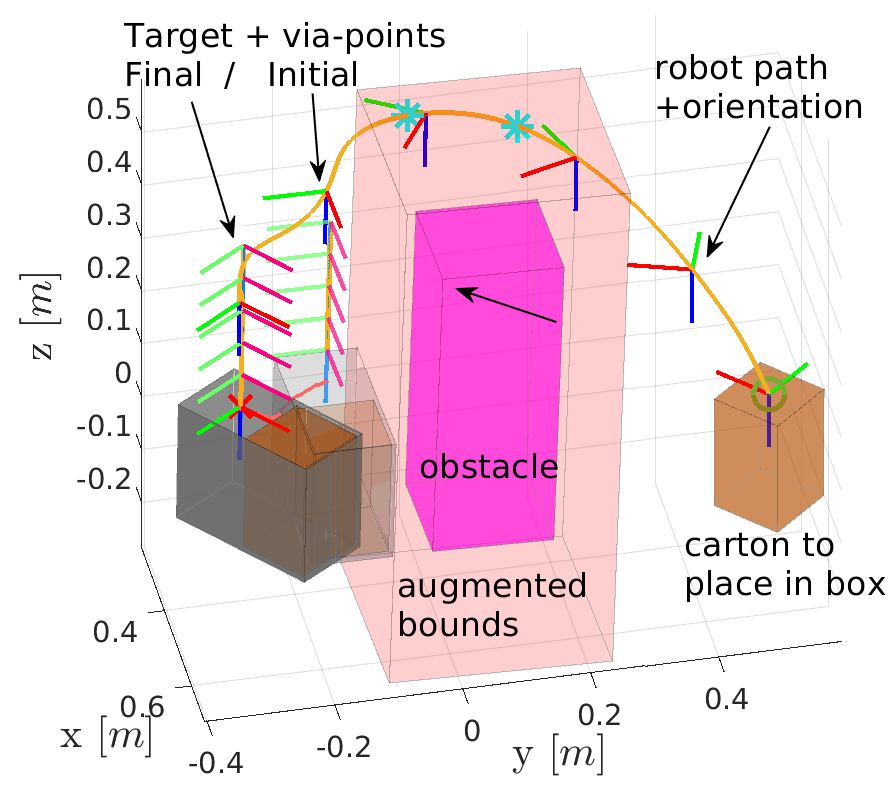}
        \caption{}
        \label{fig:c3_path}
    \end{subfigure}%

    \begin{subfigure}[b]{\linewidth}
        \centering
        \includegraphics[scale=0.52]{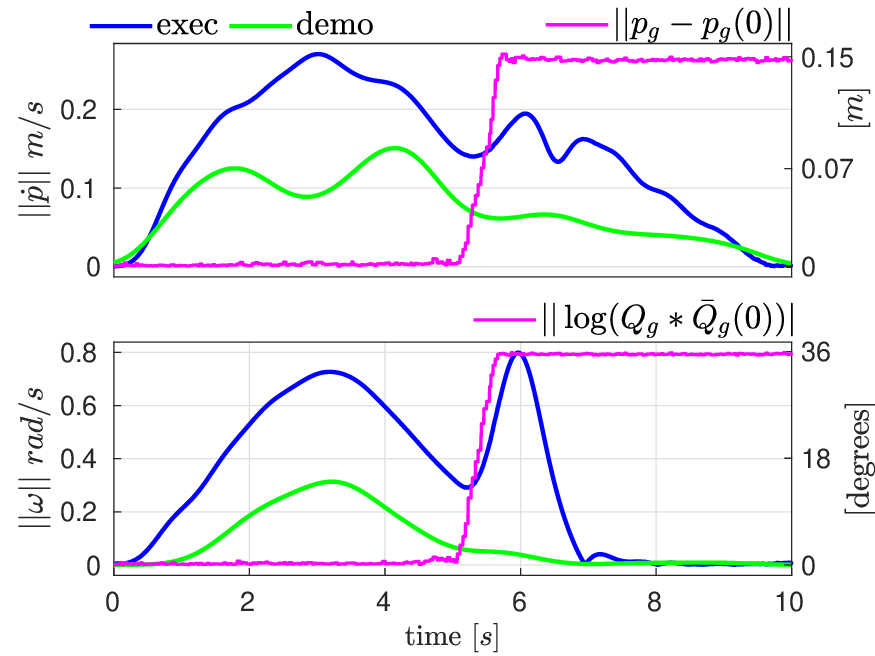}
        \caption{}
        \label{fig:c3_profile}
    \end{subfigure}%
    
    \caption{Experiment 3: Place large carton in tall box. The box is displaced online and a dynamic obstacle is introduced. (a): Robot oriented Cartesian path and visualization of the carton, box, obstacle and the associated target and obstacle-avoidance via-points. (b): Velocity and target change.}
    \label{fig:c3_results}
\end{figure}

A video with all the experimental scenarios along with explanations and visualizations can be found at: \\ 
{\small \ULurl{dropbox.com/s/51dxo7zik2v2ubj/video.mp4?dl=0}}





\section{Conclusions} \label{sec:Conclusions}

In this work we presented an improved spatial generalization for DMP that enables the incorporation of dynamic via-points by proposing an on-line adaptation scheme for the DMP weights that minimizes the distance from the learned demonstrated acceleration profile. Thus, the demonstrated motion pattern is retained and dynamic via-points and initial/final state constraints are satisfied.
Comparative simulations with other SoA methods showcase the advantages of the proposed method. Further validation is carried out via experimental packing scenarios in a dynamic environment where the incorporation of via points enables the use of a single demonstration for executing packing even when obstacles are introduced in the scene and objects and boxes of different sizes are used, whose pose can even change dynamically.


\section*{Appendix A - Unit Quaternion Preliminaries}

\noindent Given a rotation matrix  $\vect{R}\in SO(3)$, an orientation can be expressed in terms of the unit quaternion $\vect{Q} \in \mathbb{S}^{3}$ as $\vect{Q}=[w \ \vect{v}^T]^T = [\cos(\theta_2) \ \sin(\theta_2) \vect{k}^T]^T$,
where $\vect{k} \in \mathbb{S}^2$, $\theta_2 = \theta/2$ with $\theta \in [-\pi, \pi)$ are the equivalent unit axis - angle representation.
The quaternion product between the unit quaternions $\vect{Q}_1$, $\vect{Q}_2$ is denoted as $\vect{Q}_1*\vect{Q}_2$.
The inverse of a unit quaternion is equal to its conjugate which is
$\vect{Q}^{-1} = \bar{\vect{Q}} = [w \ -\vect{v}^T]^T$.
The logarithmic $\vect{\eta} = \log(\vect{Q})$ and exponential $\vect{Q} = \exp(\vect{\eta})$ mappings $\log: \ \mathbb{S}^3 \rightarrow \mathbb{R}^3$, $\exp: \ \mathbb{R}^3 \rightarrow \mathbb{S}^3$ respect the manifold's geometry and are defined as follows:
\begin{equation} \label{eq:quatLog}
    \log(\vect{Q}) \triangleq 
        \left\{
            \begin{matrix}
                2\cos^{-1}(w)\frac{\vect{v}}{||\vect{v}||}, \ |w|\ne1 \\
                [0,0,0]^T, \ \text{otherwise}
            \end{matrix}
        \right.
\end{equation}
\begin{equation} \label{eq:quatExp}
    \exp(\vect{\eta}) \triangleq 
        \left\{
            \begin{matrix}
                [\cos(||\vect{\eta}/2||), \sin(||\vect{\eta}/2||)\frac{\vect{\eta}^T}{||\vect{\eta}||}]^T, \ ||\vect{\eta}|| \ne 0 \\
                [1,0,0,0]^T, \ \text{otherwise}
            \end{matrix}
        \right.
\end{equation}

It can be found analytically that the relations between the rotational velocity and the derivative of the quaternion logarithm are:
\begin{equation} \label{eq:omega_qlogDot}
    \vect{\omega} = \vect{J}_{\eta} \dvect{\eta}
\end{equation}
\begin{equation} \label{eq:qlogDot_omega}
    \dvect{\eta} =  \vect{J}^{\dagger}_{\eta} \vect{\omega}
\end{equation}
where
\begin{align}
    \vect{J}_{\eta} &\triangleq \scalemath{0.9}{\vect{k}\vect{k}^T + \frac{\sth\cth}{\theta_2}(\vect{I}_3 - \vect{k}\vect{k}^T) +\frac{\sin^2(\theta_2)}{\theta_2}[\vect{k}]_{\times}}\label{eq:J_eta} \\
    \vect{J}^{\dagger}_{\eta} &= \scalemath{0.9}{\vect{k}\vect{k}^T + \frac{\theta_2\cth}{\sth}(\vect{I}_3 - \vect{k}\vect{k}^T) -\theta_2[\vect{k}]_{\times}} \label{eq:inv_J_eta}
\end{align}
where $[k]_{\times}$ denotes the skew-symmetric matric of $\vect{k}$.
For $\theta=0$ it can be easily verified that taking the limit of \eqref{eq:J_eta}, \eqref{eq:inv_J_eta} and using L'Hospital's rule we get $\vect{J}_{\eta} = \vect{J}^{\dagger}_{\eta} = \vect{I}_3$.

Differentiating \eqref{eq:omega_qlogDot}, \eqref{eq:qlogDot_omega} we can obtain the relations between the quaternion logarithm second time derivative and the rotational acceleration:

\begin{equation} \label{eq:omegaDot_qlogDDot}
    \dvect{\omega} = \vect{J}_{\eta} \ddvect{\eta} + \dvect{J}_{\eta} \dvect{\eta}
\end{equation}
\begin{equation} \label{eq:qlogDDot_omegaDot}
    \ddvect{\eta} =  \vect{J}^{\dagger}_{\eta} \dvect{\omega} + \dvect{J}^{\dagger}_{\eta} \vect{\omega}
\end{equation}
where
\begin{align} \label{eq:J_eta_dot}
    \dvect{J}_{\eta} &= \scalemath{0.88}{\left( 1 - \frac{\sth \cth}{\theta_2} \right)(\dvect{k}\vect{k}^T + \vect{k}\dvect{k}^T) + \frac{\sin^2(\theta_2)}{\theta_2}[\dvect{k}]_{\times}}  \nonumber \\
    &\scalemath{0.88}{+\left( \frac{1 - 2\sin^2(\theta_2)}{\theta_2} - \frac{\sth \cth}{\theta_2^2} \right)\dot{\theta}_2(\vect{I}_3 - \vect{k}\vect{k}^T)} \nonumber \\ 
    &\scalemath{0.88}{+\left( \frac{2\sth \cth}{\theta_2} - \frac{\sin^2(\theta)}{\theta_2^2} \right)\dot{\theta}_2[\vect{k}]_{\times}}
\end{align}
\begin{align} \label{eq:inv_J_eta_dot}
    \dvect{J}^{\dagger}_{\eta} &= \scalemath{0.88}{\left( 1 - \frac{\theta_2 \cth}{\sth} \right)(\dvect{k}\vect{k}^T + \vect{k}\dvect{k}^T)
    -\dot{\theta_2}[\vect{k}]_{\times}-\theta_2[\dvect{k}]_{\times}} \nonumber \\
    &\scalemath{0.88}{+\left( \frac{\sth \cth - \theta_2}{\sin^2(\theta_2)} \right)\dot{\theta}_2(\vect{I}_3 - \vect{k}\vect{k}^T)}
\end{align}
with $\dot{\theta}_2 = \frac{1}{2}\vect{k}^T \dvect{\vect{\eta}}$ and $\dvect{k} = \frac{1}{2} \left( \frac{1 - \vect{k} \vect{k}^T}{\theta_2} \dvect{\vect{\eta}} \right)$. Taking the limit of \eqref{eq:J_eta_dot}, \eqref{eq:inv_J_eta_dot} for $\theta=0$ we have that $\dvect{J}_{\eta} = \dvect{J}^{\dagger}_{\eta} = \vect{0}$.

Finally, the relations between the Cartesian torque $\vect{\tau}$ and its transformation in the quaternion logarithm space $\vect{\tau}_{\eta}$ is given by:
\begin{equation} \label{eq:logTorq_torq}
    \vect{\tau}_{\eta} =  \vect{J}_{\eta}^T \vect{\tau}
\end{equation}
\begin{equation} \label{eq:torq_logTorq}
    \vect{\tau} = (\vect{J}^{\dagger}_{\eta})^T \vect{\tau}_{\eta}
\end{equation}
These mappings follow readily from the preservation of power, i.e. it should hold that $\vect{\omega}^T\vect{\tau} = \dvect{\eta}^T\vect{\tau}_{\eta}$.


\section*{Appendix B - Recursive Least Squares derivations}

\noindent Here we provide some useful results in the form of theorems which facilitate the proof of the update formulas \eqref{eq:downdate}, \eqref{eq:update} that solve \eqref{eq:opt_prob}.
Results in Section 2 from \cite{kailath2000linear} for vectors are here extended for matrices in \textit{Theorems} $1,2$ and $5$. 
\textit{Theorems} $1,2$ are used in the proof of \textit{Theorems} $3 - 5$, and \textit{Theorems} $2 - 5$ are employed in the proof for solving \eqref{eq:opt_prob}.
 

\begin{theorem} \label{theo:LS_sol}
The solution to the problem:
\begin{align}
    \text{min}_{\vect{W}} \trace{ (\vect{Y} - \vect{W}^T \vPhi)^T \vect{R}^{-1} (\vect{Y}- \vect{W}^T \vPhi) } \label{eq:prob0}
\end{align}
with $\vect{Y} \in \mathbb{R}^{n \times m}$, $\vect{W} \in \mathbb{R}^{k \times n}$, $\vPhi \in \mathbb{R}^{k \times m}$, $\text{rank}(\vPhi) = k$, $\vect{R} \in \mathcal{S}^n_{++}$ is given by:
\begin{align} \label{eq:prob0_sol}
    \vect{W}_0 &= \vect{P}_0 \vPhi \vect{R}^{-1} \vect{Y}^T \\
    \vect{P}_0 &= (\vPhi \vect{R}^{-1} \vPhi^T)^{-1}
\end{align}
where $\vect{P}_0 > \vect{0}$.
\end{theorem}

\begin{proof}
    Taking the derivative of $f_0$ w.r.t. $\vect{W}$ and solving for $\vect{W}$, it is straightforward to verify that the solution is indeed given by $\vect{W}_0$, where $\vect{P} > \vect{0}$ since $\vect{R} > \vect{0}$ and $\text{rank}(\vPhi) = k$.
\end{proof}

\begin{theorem} \label{theo:LS_update}
The solution to the problem:
\begin{align}
    \text{min}_{\vect{W}} f_0(\vect{W}) + \trace{ (\vect{Z} - \vect{W}^T \vect{H})^T \vect{R}_1^{-1} (\vect{Z} - \vect{W}^T \vect{H}) } \label{eq:prob1}
\end{align}
with $f_0(\vect{W}) = \trace{ (\vect{Y} - \vect{W}^T \vPhi)^T \vect{R}^{-1} (\vect{Y}- \vect{W}^T \vPhi) }$, $\vect{Y} \in \mathbb{R}^{n \times m}$, $\vect{W} \in \mathbb{R}^{k \times n}$, $\vPhi \in \mathbb{R}^{k \times m}$, $\text{rank}(\vPhi) = k$, $\vect{R}, \vect{R}_1 \in \mathcal{S}^n_{++}$, $\vect{Z} \in \mathbb{R}^{n \times l}$, $\vect{H} \in \mathbb{R}^{k \times l}$, can be obtained from the solution $\vect{W}_0$, $\vect{P}_0$ of $\text{min}_{\vect{W}} f_0(\vect{W})$ as follows:
\begin{align} \label{eq:prob1_sol}
    \vect{W}_1 &= \vect{W}_0 + \vect{P}_0 \vect{H}(\vect{R}_1 + \vect{H}^T\vect{P}_0\vect{H})^{-1}(\vect{Z} - \vect{W}_0^T\vect{H})^T \\
    \vect{P}_1 \ &= \vect{P}_0 - \vect{P}_0 \vect{H}(\vect{R}_1 + \vect{H}^T\vect{P}_0\vect{H})^{-1} \vect{H}^T \vect{P}_0 \label{eq:prob1_P1}\\
    \vect{P}_1^{-1} &= \vect{P}_0 + \vect{H}\vect{R}_1^{-1}\vect{H}^T
\end{align}
where $\vect{P}_1 > \vect{0}$. 
\end{theorem}

\begin{proof}
    We can rewritte the cost function in \eqref{eq:prob1} as 
    \begin{equation*}
        \trace{ (\bvect{Y} - \vect{W}^T \bvect{\Phi})^T \bvect{R}^{-1} (\bvect{Y}- \vect{W}^T \bvect{\Phi}) }
    \end{equation*}
    with
    \begin{equation*}
        \bvect{Y} = [\vect{Y}^T \ \vect{Z}^T]^T, \ \bvect{\Phi} = [\vPhi^T \ \vect{H}^T ]^T, \ \bvect{R} = \text{blkdiag}(\vect{R}, \vect{R}_1)
    \end{equation*}
    and apply the result of \theoref{theo:LS_sol} to get:
    \begin{equation} \label{eq:prob1_temp_w}
        \vect{W} = \vect{P}_1^{-1} \bvect{\Phi} \bvect{R}^{-1} \bvect{Y}^T = \vect{P}_1^{-1}(\vPhi \vect{R}^{-1} \vect{Y}^T + \vect{H}\vect{R}_1^{-1}\vect{Z}^T)
    \end{equation}
    where $\vect{P}_1^{-1} = \bvect{\Phi} \bvect{R}^{-1} \bvect{\Phi}^T = (\vect{P}_0^{-1} + \vect{H}\vect{R}_1^{-1}\vect{H}^T)$ with $\vect{P}_0^{-1} = \vPhi \vect{R}^{-1} \vPhi^T$. It follows that $\vect{P}_1^{-1} > \vect{0}$, since $\vect{P}_0^{-1} > \vect{0}$, due to $\vect{R} > \vect{0}$ and $\text{rank}(\vPhi) = k$, and $\vect{H}\vect{R}_1^{-1}\vect{H}^T \ge 0$.
    Applying the matrix inversion lemma it follows that $\vect{P}_1$ is indeed given by \eqref{eq:prob1_P1} and substituting it in \eqref{eq:prob1_temp_w} we get:
    \begin{equation} \label{eq:prob1_temp_w2}
        \begin{aligned}
            \vect{W} =& \vect{W}_0 - \vect{P}_0 \vect{H} (\vect{R}_1 + \vect{H}^T\vect{P}_0\vect{H})^{-1}(\vect{W}_0^T\vect{H})^T + \\
           &(\vect{P}_0 - \vect{P}_0 \vect{H} (\vect{R}_1 + \vect{H}^T\vect{P}_0\vect{H})^{-1}\vect{H}^T\vect{P}_0)\vect{H}\vect{R}_1^{-1}\vect{Z}^T
        \end{aligned}
    \end{equation}
    where $\vect{W}_0 = \vect{P}_0 \vPhi \vect{R}^{-1} \vect{Y}^T$, $\vect{P}_0 = (\vPhi \vect{R}^{-1} \vPhi^T)^{-1}$ is indeed the solution of $\text{min}_{\vect{W}} f_0(\vect{W})$ based on \theoref{theo:LS_sol}.
    We can further process the last term in \eqref{eq:prob1_temp_w2}, i.e.:
    \begin{align} \label{eq:PH_y_identity}
        &\scalemath{0.9}{(\vect{P}_0 - \vect{P}_0 \vect{H} (\vect{R}_1 + \vect{H}^T\vect{P}_0\vect{H})^{-1}\vect{H}^T \vect{P}_0)\vect{H}\vect{R}_1^{-1}\vect{Z}^T } \nonumber \\
       &\scalemath{0.9}{= (\vect{P}_0\vect{H} - \vect{P}_0 \vect{H} (\vect{R}_1 + \vect{H}^T\vect{P}_0\vect{H})^{-1}\vect{H}^T \vect{P}_0\vect{H})\vect{R}_1^{-1}\vect{Z}^T } \nonumber \\
       &\scalemath{0.9}{= \vect{P}_0\vect{H}(\vect{I}_k - (\vect{R}_1 + \vect{H}^T\vect{P}_0\vect{H})^{-1}(\pm \vect{R}_1 + \vect{H}^T \vect{P}_0\vect{H}))\vect{R}_1^{-1}\vect{Z}^T } \nonumber \\
       &\scalemath{0.9}{= \vect{P}_0 \vect{H} (\vect{R}_1 + \vect{H}^T\vect{P}_0\vect{H})^{-1}\vect{Z}^T}
    \end{align}
    and substituting it back  to \eqref{eq:prob1_temp_w2} we arrive at the solution given by \eqref{eq:prob1_sol}.
\end{proof}

\begin{theorem} \label{theo:LS_equivalence}
Problem \eqref{eq:prob1} is equivalent to the problem
\begin{align}
    \text{min}_{\vect{W}} &\trace{(\vect{W} - \vect{W}_0)\vect{P}_0^{-1}(\vect{W} - \vect{W}_0)^T} + \\
    &\trace{ (\vect{Z} - \vect{W}^T \vect{H})^T \vect{R}_1^{-1} (\vect{Z} - \vect{W}^T \vect{H}) } \label{eq:prob1_equiv}
\end{align}
where $\vect{W}_0$, $\vect{P}_0$ are the solution to $\text{min}_{\vect{W}} f_0(\vect{W})$.
\end{theorem}

\begin{proof}
    Taking the gradient and solving w.r.t. $\vect{W}$ we can readily obtain \eqref{eq:prob1_temp_w} after which the same analysis as in the proof of \theoref{theo:LS_update} follows.
\end{proof}

\begin{theorem} \label{theo:LS_eq_constr}
The optimization problem
\begin{align}
    \text{min}_{\vect{W}} & f_0(\vect{W}) \label{eq:prob2} \\
    \text{s.t.} \quad & \vect{W}^T \vect{H} = \vect{Z} \nonumber
\end{align}
with $f_0(\vect{W}) = \trace{ (\vect{Y} - \vect{W}^T \vPhi)^T \vect{R}^{-1} (\vect{Y}- \vect{W}^T \vPhi) }$, $\vect{Y} \in \mathbb{R}^{n \times m}$, $\vect{W} \in \mathbb{R}^{k \times n}$, $\vPhi \in \mathbb{R}^{k \times m}$, $\text{rank}(\vPhi) = k$, $\vect{R} \in \mathcal{S}^n_{++}$, $\vect{Z} \in \mathbb{R}^{n \times l}$, $\vect{H} \in \mathbb{R}^{k \times l}$, $\text{rank}(\vect{H})=l \le k$ , is equivalent to the problem:
\begin{equation} \label{eq:prob2_equiv}
    \text{min}_{\vect{W}} f_0(\vect{W}) + \trace{(\vect{Z} - \vect{W}^T \vect{H})^T \vect{R}_{\epsilon}^{-1} (\vect{Z} - \vect{W}^T \vect{H}) }
\end{equation}
for $\vect{R}_{\epsilon} \rightarrow \vect{0}^+$.
\end{theorem}

\begin{proof}
To find the solution of \eqref{eq:prob2} we introduce the Language multipliers $\vect{V} \in \mathbb{R}^{n \times l}$ and form the Lagrangian $L(\vect{W}, \vect{V}) = f_0(\vect{W}) + 2\trace{\vect{V}^T(\vect{W}^T\vect{H} - \vect{Z})}$. The KKT conditions are:
\begin{align}
    & \partDer{L}{\vect{W}} = -\vPhi(\vect{Y} - \vect{W}^T \vPhi)^T +  \vect{H}\vect{V} = \vect{0} \label{eq:KKT_w} \\
    &\partDer{L}{\vect{V}} =  \vect{H}^T\vect{W} - \vect{Z}^T = \vect{0} \label{eq:KKT_V}
\end{align}
Solving \eqref{eq:KKT_w} for $\vect{W}$ we get:
\begin{equation} \label{eq:KKT_w2}
    \vect{W} = \vect{W}_0 - \vect{P}_0 \vect{H}\vect{V} 
\end{equation}
where $\vect{W}_0 = \vect{P}_0 \vPhi \vect{Y}^T$, $\vect{P}_0 = (\vPhi \vect{R}^{-1} \vPhi^T)^{-1}$ which is indeed invertible since $\text{rank}(\vPhi) = k$. Since $\vect{P}_0$ is invertible and $\text{rank}(\vect{H})=l$, i.e. has full column rank, we can substitute \eqref{eq:KKT_w2} in \eqref{eq:KKT_V} to solve for $\vect{V}$:
\begin{equation*}
    \vect{V} = -(\vect{H}^T\vect{P}_0\vect{H})^{-1}(\vect{Z} - \vect{W}_0^T\vect{H})^T
\end{equation*}
Substituting the last equation in \eqref{eq:KKT_w2}, we obtain the solution:
\begin{equation} \label{eq:prob2_sol}
    \vect{W}_1 = \vect{W}_0 + \vect{P}_0 \vect{H}(\vect{H}^T\vect{P}_0\vect{H})^{-1}(\vect{Z} - \vect{W}_0^T\vect{H})^T
\end{equation}
Turning our attention to problem \eqref{eq:prob2_equiv} we can apply the result of \theoref{theo:LS_update} and notice that for $\vect{R}_1 = \vect{R}_{\epsilon} \rightarrow \vect{0}^+$ we get indeed the same solution as in \eqref{eq:prob2_sol}.
\end{proof}

\begin{theorem} \label{theo:LS_downdate}
Given the solution $\vect{W}_1$, $\vect{P}_1$ of the problem
\begin{equation} \label{eq:prob3}
    \text{min}_{\vect{W}} f_0(\vect{W}) + \trace{(\vect{Z} - \vect{W}^T \vect{H})^T \vect{R}_1^{-1} (\vect{Z} - \vect{W}^T \vect{H}) }
\end{equation}
with $f_0(\vect{W}) = \trace{ (\vect{Y} - \vect{W}^T \vPhi)^T \vect{R}^{-1} (\vect{Y}- \vect{W}^T \vPhi) }$, $\vect{Y} \in \mathbb{R}^{n \times m}$, $\vect{W} \in \mathbb{R}^{k \times n}$, $\vPhi \in \mathbb{R}^{k \times m}$, $\text{rank}(\vPhi) = k$, $\vect{R}, \vect{R}_1 \in \mathcal{S}^n_{++}$, $\vect{Z} \in \mathbb{R}^{n \times l}$, $\vect{H} \in \mathbb{R}^{k \times l}$, the solution to the problem:
\begin{equation} \label{eq:prob3_b}
    \text{min}_{\vect{W}} f_0(\vect{W})
\end{equation}
can be calculated as:
\begin{align} 
    \hvect{W}_0 &= \scalemath{0.97}{\vect{W}_1 + \vect{P}_1 \vect{H}(-\vect{R}_1 + \vect{H}^T\vect{P}_1\vect{H})^{-1}(\vect{Z} - \vect{W}_1^T\vect{H})^T} \label{eq:prob3_sol_w} \\
    \hvect{P}_0 &= \vect{P}_1 - \vect{P}_1 \vect{H}(-\vect{R}_1 + \vect{H}^T\vect{P}_1\vect{H})^{-1} \vect{H}^T \vect{P}_1 \\
    \hvect{P}_0^{-1} &= \vect{P}_1^{-1} - \vect{H} \vect{R}_1^{-1} \vect{H}^T
\end{align}
where $\hvect{P}_0 > \vect{0}$.
\end{theorem}

\begin{proof}
    Denoting
    \begin{equation*}
        f_1(\vect{W}) = f_0(\vect{W}) + \trace{(\vect{Z} - \vect{W}^T \vect{H})^T \vect{R}_1^{-1} (\vect{Z} - \vect{W}^T \vect{H}) }
    \end{equation*}
    it follows that 
    \begin{equation*}
        f_0(\vect{W}) = f_1(\vect{W}) - \trace{(\vect{Z} - \vect{W}^T \vect{H})^T \vect{R}_1^{-1} (\vect{Z} - \vect{W}^T \vect{H}) }
    \end{equation*}
    We can rewrite $f_1(\vect{W})$ as:
    \begin{equation*}
        f_1(\vect{W}) = \trace{ (\bvect{Y} - \vect{W}^T \bvect{\Phi})^T \bvect{R}^{-1} (\bvect{Y}- \vect{W}^T \bvect{\Phi}) }
    \end{equation*}
    with
    \begin{equation*}
        \bvect{Y} = [\vect{Y}^T \ \vect{Z}^T]^T, \ \bvect{\Phi} = [\vPhi^T \ \vect{H}^T ]^T, \ \bvect{R} = \text{blkdiag}(\vect{R}, \vect{R}_1)
    \end{equation*}
    Minimizing $f_0(\vect{W})$ we have:
    \begin{align}
        & \partDer{f_0}{\vect{W}} = \partDer{\left(\scalemath{0.82}{f_1(\vect{W}) - \trace{(\vect{Z} - \vect{W}^T \vect{H})^T \vect{R}_1^{-1} (\vect{Z} - \vect{W}^T \vect{H}) }} \right)}{\vect{W}} = \vect{0} \nonumber \\
        &(\vect{P}_1^{-1} - \vect{H}\vect{R}_1^{-1}\vect{H}^T)\vect{W} = (\bvect{\Phi} \bvect{R}^{-1} \bvect{Y}^T - \vect{H}\vect{R}_1^{-1}\vect{Z}^T) \label{eq:prob3_temp_w0}
    \end{align}
    where $\vect{P}_1^{-1} = \bvect{\Phi} \bvect{R}^{-1} \bvect{\Phi}^T = \vect{P}_0^{-1} + \vect{H}\vect{R}_1^{-1}\vect{H}^T$ and $\vect{P}_0^{-1} = \vPhi \vect{R}^{-1} \vPhi^T$.
    Since $\vect{R} > \vect{0}$ and $\text{rank}(\vPhi) = k$ it follows that $\vect{P}_0^{-1} > \vect{0}$ which also entails that $\vect{P}_1^{-1} > \vect{0}$ as $\vect{R}_1 > \vect{0}$ and $\vect{H}\vect{R}_1^{-1}\vect{H}^T \ge \vect{0}$. Notice also that $(\vect{P}_1^{-1} - \vect{H}\vect{R}_1^{-1}\vect{H}^T) = \vect{P}_0^{-1}$ hence we can invert it in \eqref{eq:prob3_temp_w0} to get:
    \begin{align*}
        \vect{W} = \scalemath{0.93}{(\vect{P}_1^{-1} - \vect{H}\vect{R}_1^{-1}\vect{H}^T)^{-1}(\bvect{\Phi} \bvect{R}^{-1} \bvect{Y}^T + \vect{H}(-\vect{R}_1^{-1})\vect{Z}^T)}
    \end{align*}
    Applying the matrix inversion lemma we have:
    \begin{equation} \label{eq:prob3_temp_w2}
        \begin{aligned}
            \vect{W} &= \scalemath{0.95}{\vect{W}_1 - \vect{P}_1 \vect{H} (-\vect{R}_1 + \vect{H}^T\vect{P}_1\vect{H})^{-1}(\vect{W}_1^T\vect{H})^T +} \\
            &\scalemath{0.95}{(\vect{P}_1 - \vect{P}_1 \vect{H} (-\vect{R}_1 + \vect{H}^T\vect{P}_1\vect{H})^{-1}\vect{H}^T\vect{P}_1)\vect{H}(-\vect{R}_1^{-1})\vect{Z}^T}
        \end{aligned}
    \end{equation}
    where $\vect{W}_1 = \vect{P}_1 \bvect{\Phi} \bvect{R}^{-1} \bvect{Y}^T$, $\vect{P}_1 = (\bvect{\Phi} \bvect{R}^{-1} \bvect{\Phi}^T)^{-1}$ is indeed the solution of $\text{min}_{\vect{W}} f_1(\vect{W})$ based on \theoref{theo:LS_sol}.
    Notice that $-\vect{R}_1 + \vect{H}^T\vect{P}_1\vect{H}$ is in fact negative definite, hence it is invertible. To prove so, consider the following matrix:
    \begin{equation*}
        \vect{M} = 
        \begin{bmatrix}
            \vect{P}_1^{-1} & \vect{H} \\
            \vect{H}^T & \vect{R}_1
        \end{bmatrix}
    \end{equation*}
    Since $\vect{P}_1^{-1} > \vect{0}$ and the Schur complement of $\vect{M}$ w.r.t. the upper left block matrix is $\vect{P}_1^{-1} - \vect{H}\vect{R}_1^{-1}\vect{H}^T = \vect{P}_0^{-1} > \vect{0}$ it follows that $\vect{M} > \vect{0}$ \cite{strang2006linear}. Therefore, as $\vect{R}_1 > \vect{0}$ and $\vect{M} > \vect{0}$ it should also hold that the Schur complement of $\vect{M}$ w.r.t. the lower right block matrix should be positive definite, i.e.:
    \begin{equation*}
        \vect{R}_1^{-1} - \vect{H}^T \vect{P}_1 \vect{H} > \vect{0} \Rightarrow -\vect{R}_1^{-1} + \vect{H}^T \vect{P}_1 \vect{H} < \vect{0}
    \end{equation*}
    We can further process the last term in \eqref{eq:prob3_temp_w2} similar to the analysis in \eqref{eq:PH_y_identity} to find that:
    \begin{align*}
        &\scalemath{0.95}{(\vect{P}_1 - \vect{P}_1 \vect{H} (-\vect{R}_1 + \vect{H}^T\vect{P}_1\vect{H})^{-1}\vect{H}^T\vect{P}_1)\vect{H}(-\vect{R}_1^{-1})\vect{Z}^T =} \\
        &\scalemath{0.95}{\vect{P}_1 \vect{H} (-\vect{R}_1 + \vect{H}^T\vect{P}_1\vect{H})^{-1}\vect{Z}^T}
    \end{align*}
    and substituting the last in \eqref{eq:prob3_temp_w2} we obtain the result given by \eqref{eq:prob3_sol_w}. Finally, having already established that $\vect{P}_1 > \vect{0}$ and $-\vect{R}_1^{-1} + \vect{H}^T \vect{P}_1 \vect{H} < \vect{0}$ it follows that $- \vect{P}_1 \vect{H}(-\vect{R}_1 + \vect{H}^T\vect{P}_1\vect{H})^{-1} \vect{H}^T \vect{P}_1 \ge \vect{0}$ hence $\hvect{P}_0 > \vect{0}$.
\end{proof}

\section*{Declarations}

No funding was received for conducting this work and there isn't any kind of conflict of interest.



\bibliography{sn-bibliography}


\end{document}